\documentclass[11pt,english]{article}
\usepackage{mathpazo}
\usepackage[T1]{fontenc}
\usepackage[latin9]{inputenc}
\usepackage{geometry}
\usepackage{babel}

\usepackage{float}
\usepackage{calc}
\usepackage{mathtools}
\usepackage{amsmath}
\usepackage{amsthm}
\usepackage{amssymb}
\usepackage{color}
\usepackage{enumitem}
\usepackage[numbers,sort, compress]{natbib}
\usepackage{xargs}[2008/03/08]
\usepackage[colorlinks,
    linkcolor=red,
    anchorcolor=blue,
    citecolor=blue
    ]{hyperref}
\usepackage{appendix}
\allowdisplaybreaks

\makeatletter

\floatstyle{ruled}
\newfloat{algorithm}{tbp}{loa}
\providecommand{\algorithmname}{Algorithm}
\floatname{algorithm}{\protect\algorithmname}

\theoremstyle{plain}
\newtheorem{thm}{\protect\theoremname}
\theoremstyle{plain}
\newtheorem{assumption}{Assumption}
\theoremstyle{remark}
\newtheorem{rem}{Remark}
\theoremstyle{plain}

\theoremstyle{plain}

\theoremstyle{plain}
\newtheorem{lem}{Lemma}
\theoremstyle{plain}

\renewcommand{\cite}{\citep}
\renewcommand{\citet}{\citep}
\renewcommand{\citealt}{\citep}

\usepackage{algorithm}
\usepackage{algorithmicx}\usepackage{algpseudocode}\usepackage{cleveref}

\newcommand{\algmargin}{\the\ALG@thistlm}

\newlength{\whilewidth}
\algdef{SE}[parWHILE]{parWhile}{EndparWhile}[1]
{\parbox[t]{\dimexpr\linewidth-\algmargin}{%
		\hangindent\whilewidth\strut\algorithmicwhile\ #1\ \algorithmicdo\strut}}{\algorithmicend\ \algorithmicwhile}%
\algnewcommand{\parState}[1]{\State%
	\parbox[t]{\dimexpr\linewidth-\algmargin}{\strut #1\strut}}

\providecommand{\theoremname}{Theorem}

\allowdisplaybreaks

\title{Dynamic Regret of Policy Optimization in Non-stationary Environments}

\author{%
    \normalsize Yingjie Fei\thanks{School of Operations Research and Information Engineering, Cornell University; \texttt{yf275@cornell.edu}}
    \qquad
    \normalsize Zhuoran Yang\thanks{Department of Operations Research and Financial Engineering, Princeton University; \texttt{zy6@princeton.edu}} 
    \qquad
    \normalsize Zhaoran Wang\thanks{Department of Industrial Engineering and Management Sciences, Northwestern University; \texttt{zhaoranwang@gmail.com}}
    \qquad
    \normalsize Qiaomin Xie\thanks{School of Operations Research and Information Engineering, Cornell University; \texttt{qiaomin.xie@cornell.edu}}
}
\date{}

\begin{document}
\global\long\def\cT{{\cal T}}%
\global\long\def\cC{{\cal C}}%
\global\long\def\cM{{\cal M}}%
\global\long\def\cN{{\cal N}}%
\global\long\def\cV{{\cal V}}%
\global\long\def\cF{{\cal F}}%
\global\long\def\cR{{\cal R}}%
\global\long\def\cP{{\cal P}}%
\global\long\def\cG{{\cal G}}%
\global\long\def\cB{{\cal B}}%
\global\long\def\cD{{\cal D}}%
\global\long\def\cX{{\cal X}}%
\global\long\def\cS{{\cal S}}%
\global\long\def\cA{{\cal A}}%
\global\long\def\cY{{\cal Y}}%
\global\long\def\cL{{\cal L}}%
\global\long\def\cQ{{\cal Q}}%
\global\long\def\cU{{\cal U}}%
\global\long\def\real{\mathbb{R}}%
\global\long\def\E{\mathbb{E}}%
\global\long\def\J{\mathbb{J}}%
\global\long\def\P{\mathbb{P}}%
\global\long\def\pol{\pi}%
\global\long\def\indic{\mathbb{I}}%
\global\long\def\Ps{P^{\star}}%
\global\long\def\Php{\hat{P}^{\pol}}%
\global\long\def\Pp{P^{\pol}}%
\global\long\def\Vs{V^{\star}}%
\global\long\def\Vp{V^{\pol}}%
\global\long\def\Qs{Q^{\star}}%
\global\long\def\Qp{Q^{\pol}}%
\global\long\def\vhat{\hat{v}}%
\global\long\def\zhat{\hat{z}}%
\global\long\def\e{\mathbf{e}}%
\global\long\def\g{\mathbf{g}}%
\global\long\def\w{\mathbf{w}}%
\global\long\def\v{\mathbf{v}}%
\global\long\def\fmap{\phi}%
\global\long\def\pmap{\mu}%
\global\long\def\IdMat{\mathbf{I}}%
\global\long\def\bP{\mathbf{P}}%
\global\long\def\A{\mathbf{A}}%
\global\long\def\ucbrate{\lambda_{\text{UCB}}}%
\global\long\def\tmp{\beta}%
\global\long\def\diag{\mathop{\text{diag}}}%
\global\long\def\argmin{\mathop{\text{argmin}}}%
\global\long\def\argmax{\mathop{\text{argmax}}}%
\newcommandx\norm[2][usedefault, addprefix=\global, 1=\#1]{\ensuremath{\left\Vert #1\right\Vert {}_{#2}}}%
\global\long\def\conf{\text{conf}}%
\global\long\def\lse{\mathsf{lse}}%
\renewcommandx\norm[2][usedefault, addprefix=\global, 1=\#1]{\Vert#1\|_{#2}}%
\global\long\def\reg{\textup{Regret}}%
\global\long\def\dreg{\textup{D-Regret}}%
\global\long\def\kl{D_{\text{KL}}}%
\global\long\def\plen{\tau}%
\global\long\def\pnum{L}%
\global\long\def\Qchange{D_{T}}%
\global\long\def\polchange{P_{T}}%

\maketitle

\begin{abstract}
  We consider reinforcement learning (RL) in episodic MDPs with adversarial full-information 
  reward feedback and unknown fixed  transition kernels. We propose two model-free policy optimization algorithms,
  POWER and POWER++, and establish guarantees for their dynamic regret. Compared
  with the classical notion of static regret, dynamic regret is a stronger
  notion as it explicitly accounts for the non-stationarity of
  environments. The dynamic regret attained by the proposed
  algorithms interpolates between different regimes of non-stationarity,
  and moreover satisfies a notion of adaptive (near-)optimality, in
  the sense that it matches the (near-)optimal static regret under slow-changing environments.
  The dynamic regret bound features two components, one arising
  from exploration, which deals with the uncertainty of transition kernels,
  and the other arising from adaptation, which deals with non-stationary
  environments. 
  Specifically, we show that POWER++ improves over POWER on the
  second component of the dynamic regret by  actively adapting
  to non-stationarity through prediction.
  To the best of our knowledge,
  our work is the first dynamic regret analysis of model-free RL algorithms
  in non-stationary environments.
\end{abstract}

\section{Introduction} \label{sec:intro}

Classical reinforcement learning (RL) literature often evaluates an
algorithm by comparing its performance with that of the best \emph{fixed} (i.e., stationary) policy in hindsight, where the difference is 
commonly known as regret. Such evaluation metric implicitly assumes
that the environment is static so that it is appropriate to compare an algorithm 
to a single best policy. However, as we advance towards modern and
practical RL problems, we face challenges arising in dynamic and
non-stationary environments for which comparing against a single policy is no longer sufficient.

Two of the most prominent examples of RL for non-stationary environments
are continual RL \citep{kaplanis2018continual} and meta RL \citep{duan2016rl,wang2016learning}
(and more broadly meta learning \citep{finn2017model,finn2019online}),
which are central topics in the study of generalizability of RL algorithms.
In these settings, an agent encounters a stream of tasks throughout
time and aims to solve each task with knowledge accrued via solving
previous tasks. The tasks can be very different in nature from each
other, with potentially increasing difficulties. In particular, the
reward mechanism may vary across tasks, and therefore requires
the agent to adapt to the change of tasks. 
Another example of
RL under non-stationary environments is human-machine interaction~\citep{hadfield2016cooperative,radanovic2019learning}.
This line of research studies how humans and machines (or robots) should interact or collaborate to accomplish certain goals. In one scenario, a human teaches a robot to complete a task by assigning appropriate rewards to the robot but without intervening its dynamics. The rewards from the human can depend on the stage of the learning process and the rate of improvement in the robot's behaviors. Therefore, the robot has to adjust its policy over time to maximize the rewards it receives.

In the above examples, 
it is uninformative to compare an algorithm
with a fixed stationary policy, which itself may not perform well given the rapidly changing nature of environments.
It is also unclear whether existing algorithms, designed for static environments
and evaluated by the standard notion of regret, are sufficient for tackling non-stationary problems. 

We aim to address these challenges in this paper. We consider the setting
of episodic Markov decision processes (MDPs) with adversarial
full-information reward feedback and unknown fixed transition kernels. We are interested in the notion of \emph{dynamic}
regret, the performance difference between an algorithm and the set
of policies optimal for \emph{individual} episodes in hindsight. For non-stationary RL, dynamic regret
is a significantly stronger and more appropriate notion of performance measure than the standard (static) regret, but on the other hand more challenging for algorithm design and analysis.
We propose two efficient, model-free policy
optimization algorithms,  POWER and POWER++. Under a mild regularity condition of MDPs,
we provide dynamic regret analysis for both algorithms and we show that the regret bounds interpolate bewteen different regimes of non-stationarity. In particular, the bounds are of order 
$\tilde{O}(T^{1/2})$ 
when the underlying model is nearly stationary,
matching with existing near-optimal static regret bounds. In that sense, our algorithms 
are \textit{adaptively} near-optimal in slow-varying environments. To the best of our knowledge, we provide the first dynamic regret analysis for model-free RL algorithms under non-stationary
environments. 
 
Our dynamic regret bounds naturally decompose
into two terms, one due to maintaining optimism and encouraging exploration
in the face of uncertainty associated with the transition kernel, and
the other due to the changing nature of reward functions. 
This decomposition highlights the two main components an RL algorithm needs in order to perform well in non-stationary environments: effective
exploration under uncertainty and self-stabilization under drifting
reward signals. Our second algorithm, POWER++, 
takes advantage of active prediction and 
improves over POWER in terms of the second term in the dynamic regret bounds.

\paragraph{Our contributions.} 
The contributions of our work can be summarized as follows:
\begin{itemize}
    \item We propose two model-free policy optimization algorithms, POWER and POWER++, for non-stationary RL with adversarial rewards;
    \item We provide dynamic regret analysis for both algorithms, and the regret bounds are applicable across all regimes of non-stationarity of the underlying model;
    \item When the environment is nearly stationary, our dynamic regret bounds are of order $\tilde{O}(T^{1/2})$ and match the near-optimal static regret bounds, thereby demonstrating the adaptive near-optimality of our algorithms in slow-changing environments.
\end{itemize}

\paragraph{Related work.}

Dynamic regret has been considered for RL in several papers. The work of \citealt{jaksch2010near}
considers the setting of online MDP in which the transition kernel
and reward function are allowed to change $l$ times, and the regret
compares the algorithm against optimal policies for each of the $l+1$
periods. It proposes UCRL2 with restart, which  achieves
an $\tilde{O}((l+1)^{1/3}T^{2/3})$ regret where $T$ is the number
of timesteps. The work of \citet{gajane2018sliding} considers the
same setting and shows that UCRL2 with sliding windows
achieves the same regret. 
Generalizing the previous settings, the work of \citealt{ortner2019variational}
studies the setting where the changes of model is allowed to take place in every timestep. It proves that UCRL with restart achieves a regret of $\tilde{O}((B_{r}+B_{p})^{1/3}T^{2/3})$ for sufficiently large $B_{r},B_{p}>0$, 
where $B_{r}$ and $B_p$ are the variations of rewards and transition kernels over the $T$ timesteps, respectively.
The work of \citet{cheung2019reinforcement} proposes
the sliding-window UCRL2 with confidence widening, which achieves an $\tilde{O}((B_{r}+B_{p}+1)^{1/4}T^{3/4})$ regret; under additional regularity conditions, the regret can be improved to $\tilde{O}((B_{r}+B_{p}+1)^{1/3}T^{2/3})$. A Bandit-over-RL algorithm is also provided by \cite{cheung2019reinforcement} to adaptively tune the UCRL2-based algorithm to achieve an $\tilde{O}((B_{r}+B_{p}+1)^{1/4}T^{3/4})$ regret without knowing $B_r$ or $B_p$.
The work \cite{lykouris2019corruption} considers the setting of episodic MDPs in which reward functions and transition kernels get corrupted by an adversary in $K_0$ episodes. It proposes an algorithm called CRANE-RL that achieves a regret of $\tilde{O}(K_0 \sqrt{T} + K_0^2)$.  
We remark that all the work discussed so far study model-based algorithms, and we refer interested readers to  \citet{padakandla2020survey} for an excellent survey on the topic of RL in non-stationary environments. 
Dynamic regret has also been studied under the settings of multi-armed
bandits \citep{besbes2014stochastic,karnin2016multi,keskin2017chasing,luo2018efficient,auer2019adaptively,besson2019generalized,chen2019new,cheung2019learning,russac2019weighted}, 
online convex optimization 
\citep{zinkevich2003online,hall2013dynamical,besbes2015non,hall2015online,jadbabaie2015online,yang2016tracking,shahrampour2017distributed,zhang2017improved,zhang2018adaptive,zhang2018dynamic,zhao2018proximal,ravier2019prediction,roy2019multi,zhao2019bandit,zhao2019understand,zhang2020minimizing} 
and games \citep{duvocelle2018learning}. Interestingly, the notion of dynamic regret is related to the exploitability of strategies in two-player zero-sum games \citep{davis2014using}. We would also like to mention a series of papers that consider the setting of non-stationary MDPs 
\citep{dick2014online, cardoso2019large, yu2009markov, yu2009online, even2005experts, abbasi2013online, neu2010online, neu2010loop, neu2012adversarial, arora2012deterministic, rosenberg2019online, rosenberg2019shortest, jin2019learning, radanovic2019learning}, although they focus on static regret analysis.

\paragraph{Notations.}

For a positive integer $n$, we let $[n]\coloneqq\{1,2,\ldots,n\}$.
We write $x^{+}=\max\{x,0\}$ for a scalar or vector $x$, where the
maximum operator is applied elementwise.  For two non-negative sequences
$\{a_{i}\}$ and $\{b_{i}\}$, we write $a_{i}\lesssim b_{i}$ if
there exists a universal constant $C>0$ such that $a_{i}\le Cb_{i}$
for all $i$. We write $a_{i}\asymp b_{i}$ if $a_{i}\lesssim b_{i}$
and $b_{i}\lesssim a_{i}$. We use $\tilde{O}(\cdot)$ to denote $O(\cdot)$
while hiding logarithmic factors. We use $\norm[\cdot]{}$ or $\norm[\cdot]2$
to denote the $\ell_{2}$ norm of a vector or spectral norm of a matrix,
and $\norm[\cdot]1$ for the $\ell_{1}$ norm of a vector. We denote
by $\Delta(\cX)$ the set of probability distributions supported on
a discrete set $\cX$. We define 
\[
\Delta(\cX\mid\cY,H)\coloneqq\left\{ \left\{ \pi_{h}(\cdot\mid\cdot)\right\} _{h\in[H]}:\pi_{h}(\cdot\mid y)\in\Delta(\cX)\text{ for any }y\in\cY\text{ and }h\in[H]\right\} 
\]
for any set $\cY$ and horizon length $H\in\mathbb{Z}_{>0}$. 
For
$p_{1},p_{2}\in\Delta(\cX)$, we define $\kl(p_{1}\|p_{2})$ to be
the KL divergence between $p_{1}$ and $p_{2}$, that is, 
$
\kl(p_{1}\|p_{2})\coloneqq\sum_{x\in\cX}p_{1}(x)\log\left(\frac{p_{1}(x)}{p_{2}(x)}\right)$.

\section{Preliminaries \label{sec:setup}}

\subsection{Episodic MDPs and dynamic regret \label{sec:MDP}}

In this paper, we study RL in non-stationary environments via episodic MDPs
with adversarial full-information reward feedback and unknown fixed transition kernels.
An episodic MDP is defined by the state space $\cS$, the action
space $\cA$, the length $H$ of each episode, the transition kernels
$\{\cP_{h}(\cdot\mid\cdot,\cdot)\}_{h \in [H]}$ 
and the
reward functions $\{r_{h}^{k}:\cS\times\cA\to[0,1]\}_{(k,h)\in [K]\times[H]}$.
We assume that the reward functions are deterministic and potentially different across episodes, 
and that both $\cS$ and $\cA$ are discrete sets of sizes 
$S\coloneqq\left|\cS\right|$ and 
$A\coloneqq\left|\cA\right|$, respectively.

An agent interacts with the MDP through $K$ episodes without knowledge of $\{\cP_{h}\}$. At the beginning of episode
$k$, the environment provides an arbitrary state $s_{1}^{k}$ to
the agent and chooses reward functions $\{r_{h}^{k}\}_{h\in[H]}$. The choice of the reward functions is possibly
adversarial and may depend on the history of the past $(k-1)$ episodes. In step $h$ of episode $k$, the agent observes state
$s_{h}^{k}$ and then takes an action $a_{h}^{k}$, upon which the
environment transitions to the next state $s_{h+1}^{k}\sim\cP(\cdot\mid s_{h}^{k},a_{h}^{k})$.
At the same time, the environment also reveals the reward function $r_{h}^{k}$ to the agent, and the agent receives the reward
$r_{h}^{k}(s_{h}^{k},a_{h}^{k})$ (known as the full-information setting). At step $H+1$, the agent observes
state $s_{H+1}^{k}$ but does not take any action (therefore receiving
no reward), and episode $k$ is completed. We denote by $T\coloneqq KH$
the total number of steps taken throughout the $K$ episodes.

For any fixed policy $\pi=\{\pi_{h}\}_{h\in[H]}\in\Delta(\cA\mid\cS,H)$ and any $(k,h,s,a)\in[K]\times[H]\times\cS\times\cA,$
we define the value function $V_{h}^{\pi,k}:\cS\to\real$ as 
\[
V_{h}^{\pi,k}(s)\coloneqq\E_{\pi}\left[\sum_{i=h}^{H}r_{i}^{k}(s_{i},a_{i})\ \Bigg|\ s_{h}=s\right],
\]
and the corresponding action-value function $Q_{h}^{\pi,k}:\cS\times\cA\to\real$
as 
\[
Q_{h}^{\pi,k}(s,a)\coloneqq\E_{\pi}\left[\sum_{i=h}^{H}r_{i}^{k}(s_{i},a_{i})\ \Bigg|\ s_{h}=s,a_{h}=a\right].
\]
Here, the expectation
$\E_{\pi}[\cdot]$ is taken over the randomness of the state-action
tuples $\{(s_{h},a_{h},s_{h+1})\}_{h\in[H]}$, where the action $a_{h}$ is
sampled from the policy $\pi_{h}(\cdot\mid s_{h})$ and the next state
$s_{h+1}$ is sampled from the transition kernel $\cP_{h}(\cdot\mid s_{h},a_{h})$.
The Bellman equation is given by 
\begin{equation}
\label{eq:bellman}
Q_{h}^{\pi,k}(s,a)  =r_{h}^{k}+\P_{h}V_{h+1}^{\pi,k},\qquad
V_{h}^{\pi,k}(s) \coloneqq\left\langle Q_{h}^{\pi,k},\pi_{h}\right\rangle _{\cA},\qquad V_{H+1}^{\pi}(s)=0.
\end{equation}
In Equation \eqref{eq:bellman}, we use $\left\langle \cdot,\cdot\right\rangle _{\cA}$
to denote the inner product over $\cA$ and we will omit the subscript
$\cA$ in the sequel when appropriate;  we also define the operator
\[
(\P_{h}f)(s,a)\coloneqq\E_{s'\sim\cP_{h}(\cdot\ |\ s,a)}[f(s')]
\]
for any function $f:\cS\to\real$.

Under the setting of episodic MDPs, the agent aims to approximate the optimal non-stationary
policy by interacting with the environment. Let $\pi^{*,k}=\argmax_{\pi\in\Delta(\cA\mid\cS,H)}V_{1}^{\pi,k}(s_{1}^{k})$
be the optimal policy of episode $k$, and suppose that the agent
executes policy $\pi^{k}$ in episode $k$. The difference in values
between $V_{1}^{\pi^{k},k}(s_{1}^{k})$ and $V_{1}^{\pi^{*,k},k}(s_{1}^{k})$
serves as the regret or the sub-optimality of the agent's
policy $\pi^{k}$ in episode $k$. Therefore, the \textit{dynamic regret} for $K$ episodes is defined as
\begin{align}
\dreg(K) & \coloneqq\sum_{k\in[K]}\left[V_{1}^{\pi^{*,k},k}(s_{1}^{k})-V_{1}^{\pi^{k},k}(s_{1}^{k})\right].\label{eq:dynamic_regret}
\end{align}
Dynamic regret is a stronger notion than the classical regret
measure found in the literature of online learning and reinforcement
learning, which is also known as static regret and defined as 
\begin{equation}
\reg(K)\coloneqq\sum_{k\in[K]}\left[V_{1}^{\pi^{*},k}(s_{1}^{k})-V_{1}^{\pi^{k},k}(s_{1}^{k})\right],\label{eq:static_regret}
\end{equation}
where $\pi^{*}=\argmax_{\pi\in\Delta(\cA\mid\cS,H)}\sum_{k\in[K]}V_{1}^{\pi,k}(s_{1}^{k})$.
In words, dynamic regret compares the agent's policy to the optimal policy
of \emph{each individual} episode in the hindsight, while static
regret compares the agent's policy to only the optimal fixed policy over
all episodes combined. Therefore, the notion of dynamic regret is
a more natural measure of performance under non-stationary environments.
It is clear that  dynamic regret always upper bounds   static regret: 
\begin{align*}
\dreg(K) & =\sum_{k\in[K]}\left[\max_{\pi\in\Delta(\cA\mid\cS,H)}V_{1}^{\pi,k}(s_{1}^{k})-V_{1}^{\pi^{k},k}(s_{1}^{k})\right]\\
& \ge\max_{\pi\in\Delta(\cA\mid\cS,H)}\sum_{k\in[K]}\left[V_{1}^{\pi,k}(s_{1}^{k})-V_{1}^{\pi^{k},k}(s_{1}^{k})\right]
 =\reg(K).
\end{align*}
When $\{\pi^{*,k}\}$ happen to be identical for all episodes $k\in[K]$, dynamic
regret reduces to static regret. 

\subsection{Model assumptions}

For any policy $\pi$, step $h\in[H]$ and states $s,s'\in\cS$, we
denote by $\cP_{h}^{\pi}(s'\ |\ s)$ the probability of transitioning
from $s$ to $s'$ in step $h$ when policy $\pi$ is executed, i.e.,
$\cP_{h}^{\pi}(s'\ |\ s)\coloneqq\sum_{a\in\cA}\cP_{h}(s'\ |\ s,a)\cdot\pi_{h}(a\ |\ s)$. 
The quantity 
$\cP_{h}^{\pi}$ 
is also known as the visitation measure of $\pi$ at state $s$ and step $h$.
For any pair of policies $\pi$ and $\pi'$, we define the shorthands
\begin{align*}
\norm[\pi_h-\pi'_h]{\infty} & \coloneqq\max_{s\in\cS}\norm[\pi_h(\cdot\ |\ s)-\pi'_h(\cdot\ |\ s)]1,\\
\norm[\cP_{h}^{\pi}-\cP_{h}^{\pi'}]{\infty} & \coloneqq\max_{s\in\cS}\norm[\cP_{h}^{\pi}(\cdot\ |\ s)-\cP_{h}^{\pi'}(\cdot\ |\ s)]1.
\end{align*}
The following assumption stipulates that the visitation measures are smooth
with respect to policies.
\begin{assumption}[Smooth visitation measures]
	\label{asp:visit_measure_smooth} We assume that there exists a universal
	constant $C>0$ such that $\norm[\cP_{h}^{\pi}-\cP_{h}^{\pi'}]{\infty}\le C\cdot\norm[\pi_{h}-\pi'_{h}]{\infty}$ for all $h\in[H]$ and all
	pairs of policies $\pi,\pi'$.
\end{assumption}
Assumption \ref{asp:visit_measure_smooth} states that the visitation
measures do not change drastically when similar policies
are executed. 
This notion of smoothness in visitation measures also appears in \citet{radanovic2019learning} in the
context of two-player games.
\begin{rem}
	Assumption \ref{asp:visit_measure_smooth} can in fact be relaxed
	to $\norm[\cP_{h}^{\pi}-\cP_{h}^{\pi'}]{\infty}\le C\cdot\norm[\pi_{h}-\pi'_{h}]{\infty}$
	for all $h\in[H]$ and $C=O(T^{\alpha})$ that holds for all $\alpha>0$ (i.e., the Lipschitz parameter $ C $ is sub-polynomial in $T$), and our
	algorithms and results remain the same. We choose to instead require
	$C>0$ to be a universal constant for clear exposition. 
\end{rem}

Next, we introduce several measures of changes in MDPs and
algorithms. Define 
\begin{equation}
\polchange\coloneqq\sum_{k\in[K]}\sum_{h\in[H]}\norm[\pi_{h}^{*,k}-\pi{}_{h}^{*,k-1}]{\infty},\label{eq:opt_policy_max_var}
\end{equation}
where we set $\pi_{h}^{*,0}=\pi_{h}^{*,1}$ for $h\in[H]$. Note
that $\polchange$ measures the total variation in the optimal policies
of adjacent episodes. Oftentimes, algorithms are designed to estimate
the optimal policies
$\{\pi^{*,k}\}_{k\in[K]}$ 
by estimating action-value functions
$\{Q^{\pi^{*,k},k}\}_{k\in[K]}$
via iterates
\{$Q^{k}\}_{k\in[K]}$. 
For such algorithms,
we define 
\begin{equation}
\Qchange\coloneqq\sum_{k\in[K]}\sum_{h\in[H]}\max_{s\in\cS}\norm[Q_{h}^{k}(s,\cdot)-Q_{h}^{k-1}(s,\cdot)]{\infty}^{2},\label{eq:Q_change}
\end{equation}
where we set $Q_{h}^{0}=Q_{h}^{1}$ for $h\in[H]$.
Therefore, the quantity $\Qchange$ computes total variation in algorithmic iterates $\{Q^{k}\}$. 
The notions of $\polchange$ and $\Qchange$
are also used in the work of \citet{besbes2015non,hall2013dynamical,hall2015online,rakhlin2013optimization,zinkevich2003online} and are known as \textit{variation budgets} or \textit{path lengths}.
We assume that we have access to quantities $\polchange$ and $\Qchange$
or their upper bounds via an oracle, but we do not know $\{\pi^{*,k}\}$. 
Such assumptions are standard in non-stationary RL and online convex optimization
\citep{rakhlin2012online,rakhlin2013optimization,besbes2015non,jaksch2010near,gajane2018sliding,ortner2019variational}. 

\subsection{Connections with popular RL paradigms} \label{sec:applications}
We briefly discuss how the setting introduced in Section \ref{sec:MDP} is related to several popular paradigms of RL. In certain settings of continual and meta RL, an agent needs to solve tasks one after another in the same physical environment and receives  rewards for each task commensurate to the agent's performance in solving the task. A task can therefore be seen as an episode in our episodic setting. Since the tasks are presented and solved within the same physical environment, it is sufficient to assume a fixed transition model as we do in Section \ref{sec:MDP}. On the other hand, the tasks to be solved by the agent can be substantially different from each other in reward mechanism, as such detail of each task is potentially determined by the agent's performance in all previous tasks. This suggests that the rewards of the tasks are possibly non-stationary, corresponding to the quantities $\{r^{k}_{h}\}$ in our setting.

Our setting can also be viewed as a high-level abstraction for human-machine interaction. As in the example discussed in Section \ref{sec:intro}, a human guides a robot (the learner) to accomplish certain tasks by only presenting rewards according to the performance of the robot. Here, we can think of the period in between two presented rewards as an episode in our setting. We may also set the physical state of the robot as the state of our model, thus implying a fixed state transition from the robot's perspective. Moreover, the rewards are controlled by the human in a way that possibly depends on time and  history of the robot's performance, which corresponds to our assumption on $\{r^{k}_{h}\}$.



\section{Algorithms}
In this section, we present two efficient and model-free  algorithms: \textbf{P}olicy \textbf{O}ptimization \textbf{W}ith P\textbf{E}riodic \textbf{R}estart (POWER) and its enhanced version, POWER++. Let us introduce some additional notations before proceeding. We set $d=\left|\cS\right|\left|\cA\right|$,
and let $\fmap(s,a)$ be the canonical basis of $\real^{d}$ corresponding to the state-action pair $(s,a)\in\cS\times\cA$: that is, the $(s',a')$-th entry of $\fmap(s,a)$  equals to 1 if $(s,a) = (s',a')$ and 0 otherwise.

\subsection{POWER}

		
		
		
		
		
		
		
		
		
		
		
	

We present our first algorithm, POWER, in Algorithm \ref{alg:dyn_oppo}. 
Algorithm \ref{alg:dyn_oppo} is inspired by the work of \citet{cai2019provably,efroni2020optimistic}.
It mainly consists of a policy update and a policy evaluation step.
The policy update step in Line \ref{line:dyn_oppo_update_policy}
is equivalent to solving the following optimization problem:
\begin{equation}
\pi^{k}=\argmax_{\pi\in\Delta(\cA\ |\ \cS,H)}L_{k-1}(\pi)-\frac{1}{\alpha}\E_{\pi^{k-1}}\left[\sum_{h\in[H]}\kl(\pi_{h}(\cdot\ |\ s_{h})\|\pi_{h}^{k-1}(\cdot\ |\ s_{h}))\ \Bigg|\ s_{1}=s_{1}^{k}\right],\label{eq:policy_update_optimize}
\end{equation}
where 
\begin{align*}
L_{k-1}(\pi) & \coloneqq V_{1}^{\pi^{k-1},k-1}(s_{1}^{k})\\
& \quad+\E_{\pi^{k-1}}\left[\sum_{h\in[H]}\left\langle Q_{h}^{\pi^{k-1},k-1}(s_{h},\cdot),\pi_{h}(\cdot\ |\ s_{h})-\pi_{h}^{k-1}(\cdot\ |\ s_{h})\right\rangle \ \Bigg|\ s_{1}=s_{1}^{k}\right]
\end{align*}
is a local linear approximation of $V_{1}^{\pi,k-1}(s_{1}^{k})$ at
$\pi = \pi^{k-1}$. In view of Equation \eqref{eq:policy_update_optimize},
we observe that 
the policy update step can be seen as a mirror descent
(MD) step with KL divergence as the Bregman divergence. The policy
evaluation step in Line \ref{line:dyn_oppo_policy_eval} estimates
value functions of each step. To that end, it invokes  a subroutine, EvaluatePolicy, which computes the intermediate estimates 
$w_{h}^{k}$ as the solution of the following regularized
least-squares problem 
\[
w_{h}^{k}\leftarrow\argmin_{w\in\real^{d}}\sum_{t\in[k-1]}(V_{h+1}^{k}(s_{h+1}^{t})-\fmap(s_{h}^{t},a_{h}^{t})^{\top}w)^{2}+\lambda\cdot\norm[w]2^{2}.
\]
This step can be efficiently computed by taking the sample mean of $\{V_{h+1}^{k}(s_{h+1}^{t})\}_{t\in[k-1]}$. In fact, one has 
\[
w_{h}^{k}(s,a) = \fmap(s,a)^\top w_{h}^{k} = \sum_{s'\in\cS} \frac{N_{h}^{k}(s,a,s')}{N_{h}^{k}(s,a)+\lambda}\cdot V_{h+1}^{k}(s'),
\]
for each $(s,a)$, where the function $N_{h}^{k}$ counts the number of times each tuple $(s,a,s')$ or $(s,a)$ has been visited by the algorithm at step $h$ prior to episode $k$.
To facilitate exploration in the face of uncertainties, EvaluatePolicy additionally defines a bonus term $\Gamma^{k}_{h}(s,a)\propto [N^{k}_{h}(s,a)]^{-1/2}$ for each state-action pair $(s,a)$.
The estimated action-value function is then set as $Q_{h}^{k}=r^{k}_{h} + w^{k}_{h} + \Gamma^{k}_{h} $.
We provide the detailed implementation of the subroutine EvaluatePolicy in Algorithm \ref{alg:eval_pol} in Appendices.

In addition to updating and evaluating policy, Algorithm \ref{alg:dyn_oppo}
features a periodic restart mechanism, which resets its policy estimate
every $\plen$ episodes. Restart mechanisms have been used to handle
non-stationarity in RL \citep{jaksch2010near,ortner2019variational} and related problems including bandits \citep{besbes2014stochastic},
online convex optimization \citep{besbes2015non,jadbabaie2015online}
and games \citep{duvocelle2018learning,radanovic2019learning}. Intuitively,
by employing the restart mechanism, Algorithm \ref{alg:dyn_oppo}
is able to stabilize its iterates against  non-stationary drift in the learning process 
due to adversarial reward functions.
We remark that our Algorithm \ref{alg:dyn_oppo} is very different from those used in the existing non-stationary RL literature. Notably, Algorithm \ref{alg:dyn_oppo} is model-free, which is more efficient than the model-based algorithms proposed in e.g., \cite{jaksch2010near, ortner2019variational, gajane2018sliding, cheung2019learning, lykouris2019corruption}, with respect to both time and space complexities.

\begin{algorithm}[t]
	\begin{algorithmic}[1]
		
		\Require Confidence level $\delta$, number of episodes $K$, restart cycle length $\plen$, regularization factor $\lambda$ and bonus multiplier $\beta$
		
		
		\For{episode $k=1,\ldots,K$}
		
		\State Receive the initial state $s_{1}^{k}$
		
		\If{ $k\mod\plen=1$ }\Comment{periodic restart}
		
		\State Set $\{Q_{h}^{k-1}\}_{h\in[H]}$ as zero functions and $\{\pi_{h}^{k-1}\}_{h\in[H]}$
		as uniform distributions on $\cA$ \label{line:dyn_oppo_restart}
		
		\EndIf
		
		\For{ step $h=1,2,\ldots,H$} \Comment{policy update}
		
		\State Update the policy by $\pi_{h}^{k}(\cdot\ |\ \cdot)\propto\pi_{h}^{k-1}(\cdot\ |\ \cdot)\cdot\exp\{\alpha\cdot Q_{h}^{k-1}(\cdot,\cdot)\}$
		\label{line:dyn_oppo_update_policy}
		
		\State Take action $a_{h}^{k}\sim\pi_{h}^{k}(\cdot\ |\ s_{h}^{k})$
		
		\State Observe the reward function $r_{h}^{k}(\cdot,\cdot)$ and
		receive the next state $s_{h+1}^{k}$
		
		\EndFor
		
		\State Compute $\{Q_{h}^{k}\}$ by $\text{EvaluatePolicy}(k,\{r_{h}^{k}\},\{\pi_{h}^{k}\}, \lambda, \beta)$
		\Comment{policy evaluation} \label{line:dyn_oppo_policy_eval}
		
		\EndFor
		
	\end{algorithmic}
	
	\caption{
	POWER \label{alg:dyn_oppo}}
\end{algorithm}

\subsection{POWER++}

Instead of only passively tackling non-stationarity, we may enhance our algorithms
with active prediction of the environment. Optimistic mirror descent (OMD) provides 
exactly such prediction functionality via the so-called predictable sequences. It is well-known in the online learning literature that 
OMD provides improved regret guarantees than
MD algorithm \citep{rakhlin2012online,rakhlin2013optimization}. First proposed by \citealt{nemirovski2004prox} under the
name ``mirror-prox'', OMD maintains a sequence of main and intermediate iterates. 
Through the predictable
sequences in intermediate iterates, it exploits certain structures of the problem at hand, and
therefore achieve better theoretical guarantees.
We incorporate predictable sequences into POWER and arrive at an enhanced algorithm, POWER++, which is presented in Algorithm \ref{alg:acc_oppo}. 

\begin{algorithm}[t]
	\begin{algorithmic}[1]
		
		\Require Confidence level $\delta$, number of episodes $K$, restart cycle length $\plen$, regularization factor $\lambda$ and bonus multiplier $\beta$
		
		
		\State Set $\{r_{h}^{0}\}_{h\in[H]}$ as zero functions
		
		\For{episode $k=1,\ldots,K$}
		
		\State Receive the initial state $s_{1}^{k}$
		
		\If{ $k\mod\plen=1$ }\Comment{periodic restart}
		
		\State Set $\{Q_{h}^{k-1}\}_{h\in[H]}$ as zero functions and $\{\pi_{h}^{k-1}\}_{h\in[H]}$
		as uniform distributions on $\cA$
		
		\EndIf
		
		\For{ step $h=1,2,\ldots,H$} \Comment{intermediate policy update}
		
		\State Update the policy by $\pi_{h}^{k-1/2}(\cdot\ |\ \cdot)\propto\pi_{h}^{k-1}(\cdot\ |\ \cdot)\cdot\exp\{\alpha\cdot Q_{h}^{k-1}(\cdot,\cdot)\}$
		\label{line:acc_oppo_update_policy_mirror}
		
		\EndFor
		
		\State Compute $\{Q_{h}^{k-1/2}\}$ by $\text{EvaluatePolicy}(k,\{r_{h}^{k-1}\},\{\pi_{h}^{k-1/2}\}, \lambda, \beta)$
		\label{line:acc_oppo_policy_eval_mirror}
		
		\Comment{intermediate policy evaluation} 
		
		\For{ step $h=1,2,\ldots,H$} \Comment{main policy update}
		
		\State Update the policy by $\pi_{h}^{k}(\cdot\ |\ \cdot)\propto\pi_{h}^{k-1}(\cdot\ |\ \cdot)\cdot\exp\{\alpha\cdot Q_{h}^{k-1/2}(\cdot,\cdot)\}$
		\label{line:acc_oppo_update_policy}
		
		\State Take action $a_{h}^{k}\sim\pi_{h}^{k}(\cdot\ |\ s_{h}^{k})$
		
		\State Observe the reward function $r_{h}^{k}(\cdot,\cdot)$ and
		receive the next state $s_{h+1}^{k}$
		
		\EndFor
		
		\State Compute $\{Q_{h}^{k}\}$ by $\text{EvaluatePolicy}(k,\{r_{h}^{k}\},\{\pi_{h}^{k}\}, \lambda, \beta)$
		\Comment{main policy evaluation} \label{line:acc_oppo_policy_eval}
		
		\EndFor
		
	\end{algorithmic}
	
	\caption{
	POWER++
	\label{alg:acc_oppo}}
\end{algorithm}

In Algorithm \ref{alg:acc_oppo}, Lines \ref{line:acc_oppo_update_policy_mirror}
and \ref{line:acc_oppo_update_policy} together form the OMD steps.
Line \ref{line:acc_oppo_policy_eval_mirror} estimates the intermediate
action-value function $Q_{h}^{k-1/2}$ to be used in the second OMD
step (Line \ref{line:acc_oppo_update_policy}). The series of iterates
$\{Q_{h}^{k-1}\}$ in Line \ref{line:acc_oppo_update_policy_mirror}
is the so-called predictable sequence in OMD. Note that we do not
execute the intermediate policy $\pi^{k-1/2}$ in the first (and intermediate) OMD step
(Line \ref{line:acc_oppo_update_policy_mirror}), which is only used to compute the intermediate value
estimates $\{V_{h}^{k-1/2}\}$. Rather, we execute the policy $\pi^{k}$
updated by the second (and main) OMD step.  
Finally, we remark that both Algorithms \ref{alg:dyn_oppo} and \ref{alg:acc_oppo}
have polynomial space and time complexities in $S$, $A$ and $T$.

\section{Main results }

To help with the presentation of our main results, we define the thresholding operator $\Pi_{[a,b]}(x)\coloneqq\max\{\min\{x,b\},a\}$
and we adopt the convention that $x/0=\infty$ for $ x \in \real $. We also define $\pnum\coloneqq\left\lceil \frac{K}{\plen}\right\rceil $ to be the number of restarts that take place in Algorithm \ref{alg:dyn_oppo} or \ref{alg:acc_oppo}. The following theorem
gives an upper bound for the dynamic regret incurred by Algorithm
\ref{alg:dyn_oppo}. 
\begin{thm}[Upper bound for Algorithm \ref{alg:dyn_oppo}]
	\label{thm:regret_dyn_oppo} \emph{} Under Assumption \ref{asp:visit_measure_smooth},
	for any $\delta\in(0,1]$, with probability at least $1-\delta$ and the choice of $\lambda=1$,
	$\alpha=\sqrt{\frac{L\log A}{KH^{2}}}$, $\plen=\Pi_{[1,K]}\left(\left\lfloor \left(\frac{T\sqrt{\log A}}{H\polchange}\right)^{2/3}\right\rfloor \right)$
	and $\beta=C_{\beta}H\sqrt{S\log(dT/\delta)}$ (for some universal
	constant $C_{\beta}>0$) in Algorithm \ref{alg:dyn_oppo}, the dynamic
	regret of Algorithm \ref{alg:dyn_oppo} is bounded by
	\begin{align*}
	\dreg(K) & \lesssim\sqrt{H^{3}S^{2}AT\!\cdot\!\log^{2}(dT\!/\!\delta)}\!+\!\begin{cases}
	\!\sqrt{H^{3}T\log A}, & \text{if }0\le\polchange\le\sqrt{\frac{\log A}{K}},\\
	\!\left(H^{2}T\sqrt{\log A}\right)^{2/3}\polchange^{1/3}, & \text{if }\sqrt{\frac{\log A}{K}}\!\le\!\polchange\!\lesssim\! K\sqrt{\log A},\\
	\!H^{2}\polchange, & \text{if }\polchange\gtrsim K\sqrt{\log A}.
	\end{cases}
	\end{align*}
	The result also holds if we replace $\polchange$ in the above with
	its upper bound. When the upper bounds on $\dreg(K)$ exceed $T$,
	we have $\dreg(K)\le T$.
\end{thm}
The proof is given in Appendix \ref{sec:proof_regret_dyn_oppo}.
    The regret bound in Theorem \ref{thm:regret_dyn_oppo}
    interpolates smoothly throughout three regimes of $\polchange$:
    \begin{itemize}
    	\item Small $\polchange$: when $0\le\polchange\le\sqrt{\frac{\log A}{K}}$,
    	the dynamic regret scales as $\tilde{O}(T^{1/2})$ and subsumes the static regret results in \citet{cai2019provably,efroni2020optimistic} under the full-information setting. 
    	In view of \cite{azar2017minimax}, this bound is also nearly optimal (up to polynomial factors of $H$, $S$ and $A$).
    	Therefore, our bound in Theorem \ref{thm:regret_dyn_oppo} is \textit{adaptively} near-optimal under small $\polchange$; 
    	\item Moderate $\polchange$: when $\sqrt{\frac{\log A}{K}}\le\polchange\lesssim K\sqrt{\log A}$,
    	we obtain a dynamic regret of order $\tilde{O}(T^{2/3}\polchange^{1/3})$,
    	which is $\tilde{O}(T^{2/3})$ if $\polchange=O(1)$ and sub-linear in $T$
    	if $\polchange=o(K)$. Similar $\tilde{O}(T^{2/3})$ bounds have been
    	achieved by model-based algorithms in 
    	\citet{jaksch2010near,gajane2018sliding,ortner2019variational,cheung2019reinforcement},
    	which are less efficient than our model-free algorithms in both time and space complexities;
    	\item Large $\polchange$: when $\polchange\gtrsim K\sqrt{\log A}$, the
    	model is highly non-stationary and Algorithm \ref{alg:dyn_oppo} incurs a linear
    	regret in $T$. 
    \end{itemize}

In addition, the dynamic regret bound in Theorem \ref{thm:regret_dyn_oppo}
can be seen as a combination of two parts. The first 
is the cost paid for being optimistic and due to sum of bonus terms
$\{\Gamma_{h}^{k}\}$ in Algorithm \ref{alg:eval_pol} (see Equation \eqref{eq:model_pred_err_bound}
in the proof for details).
This part is necessary to enforce optimism
in the face of uncertainty generated by the transition kernels and
is key to effective exploration. The second part is the error
caused by non-stationarity of reward functions and depends on $\polchange$.
Such decomposition is
not available in the dynamic regret analysis of online convex optimization
problems where MD/OMD-based algorithms have been widely applied. In particular, the dynamic regret bound for online optimization
lacks the term due to bonus as it does not require exploration, which
is nevertheless a key component underlying RL algorithms that provably explore.

Next we present a result for Algorithm \ref{alg:acc_oppo}.
\begin{thm}[Upper bound for Algorithm \ref{alg:acc_oppo}]
	\label{thm:regret_acc_oppo} \emph{} Under Assumption  \ref{asp:visit_measure_smooth}, for any $\delta\in(0,1]$,
	with probability at least $1-\delta$ and the choice of $\lambda=1$,
	$\alpha=\sqrt{\frac{\pnum H\log A}{\Qchange}}$, $\plen=\Pi_{[1,K]}\left(\left\lfloor \left(\frac{\sqrt{\Qchange\cdot T\log A}}{H^{2}\polchange}\right)^{2/3}\right\rfloor \right)$
	and $\beta=C_{\beta}H\sqrt{S\log(dT/\delta)}$ (for some universal
	constant $C_{\beta}>0$) in Algorithm \ref{alg:acc_oppo}, the dynamic
	regret of Algorithm \ref{alg:acc_oppo} is bounded by 
	\begin{align*}
	\dreg(K) & \!\lesssim\!\!\sqrt{\!H^{3}S^{2}AT\!\cdot\!\log^{2}(dT/\delta)}\!+\!\begin{cases}
	\!\sqrt{\Qchange\cdot H\log A}, & \text{if }0\le\polchange\le\sqrt{\frac{\Qchange\cdot\log A}{K^{2}H^{3}}},\\
	\!\left(H\!\sqrt{\Qchange\!\cdot\! T\log\! A}\right)^{2/3}\!\polchange^{1/3}, & \text{if }\sqrt{\frac{\Qchange\cdot\log\! A}{K^{2}H^{3}}}\!\le\!\polchange\!\lesssim\!\frac{\sqrt{\Qchange\cdot T\log\! A}}{H^{2}},\\
	\!H^{2}\polchange, & \text{if }\polchange\gtrsim\frac{\sqrt{\Qchange\cdot T\log A}}{H^{2}}.
	\end{cases}
	\end{align*}
	The result also holds if we replace $\polchange$ and $\Qchange$
	in the above with their upper bounds. When the upper bounds on $\dreg(K)$
	exceed $T$, we have $\dreg(K)\le T$.
\end{thm}
The proof is given in Appendix \ref{sec:proof_regret_acc_oppo}. A few remarks about Theorem \ref{thm:regret_acc_oppo} are in order.
	Similar to Theorem \ref{thm:regret_dyn_oppo}, the result in Theorem
	\ref{thm:regret_acc_oppo} interpolates across three regimes
	depending on the magnitude of $\polchange$, and decomposes into two terms respectively arising from the uncertainties of 
	transition kernels and non-stationarity of reward functions. 
	Moreover, thanks to
	the OMD steps in Algorithm \ref{alg:acc_oppo} that actively make
	predictions via predictable sequence $\{Q_{h}^{k-1}\}$, 
	the bound in Theorem \ref{thm:regret_acc_oppo} is strictly better than that in Theorem \ref{thm:regret_dyn_oppo} in view of the fact that $\Qchange \lesssim KH^{3}$. 
	When $\polchange$
	is moderate, i.e., $\sqrt{\frac{\Qchange\cdot\log A}{K^{2}H^{3}}}\le\polchange\lesssim\frac{\sqrt{\Qchange\cdot T\log A}}{H^{2}}$,
	the dynamic regret bound in Theorem \ref{thm:regret_acc_oppo} is
	of order $\tilde{O}(T^{1/3}\Qchange^{1/3}\polchange^{1/3})$, which
	is similar to the result of \citet[Theorem 3]{jadbabaie2015online}
	obtained for online optimization problems. Regret bounds that depend on $\Qchange$,
	the variation of predictable sequences, have also appeared in \citet{rakhlin2012online,rakhlin2013optimization},
	although for static regret and online optimization problems.
%



\paragraph{Technical highlights.}

A central step of our dynamic regret analysis is to control the expected performance
difference between the estimated policies $\{\pi^{k}\}$ and the optimal $\{\pi^{*,k}\}$, defined
as 
\[
\sum_{k\in[K]}\sum_{h\in[H]}\E_{\pi^{*,k}}\left[\left\langle Q_{h}^{k}(s_{h},\cdot),\pi_{h}^{*,k}(\cdot\ |\ s_{h})-\pi_{h}^{k}(\cdot\ |\ s_{h})\right\rangle \ \Bigg|\ s_{1}=s_{1}^{k}\right].
\]
Note the the expectation is taken over $\{\pi^{*,k}\}$ which may vary over episodes $k$.
For static regret, i.e., when $\pi^{*,k}\equiv\pi^{*}$ for $k\in[K]$,
we may control the above term by a standard telescoping argument, which is not viable for dynamic regret analysis.
Instead, we decompose the above expectation  into $\E_{\pi^{*,k}}[\cdot]=\E_{\pi^{*,k_{0}}}[\cdot] + \E_{\pi^{*,k}-\pi^{*,k_{0}}}[\cdot]$.
Here, $k_{0}<k$ is the episode in which restart takes place most
recently prior to episode $k$. 
The first expectation $\E_{\pi^{*,k_{0}}}[\cdot]$ is taken over $\pi^{*,k_{0}}$, which stays constant for the period from $k_0$ to the next restart. Therefore, we may apply a customized telescoping argument to each period between restarts. 
The second expectation $\E_{\pi^{*,k}-\pi^{*,k_{0}}}[\cdot]$ from the decomposition involves the difference $\pi^{*,k}-\pi^{*,k_{0}}$ and can be bounded by $\polchange$. 
See Lemmas \ref{lem:perf_diff_bound_fixed_expect} and \ref{lem:perf_diff_bound_varying_policies} in Appendices, respectively, 
for details of controlling the two expectations. 
Furthermore, it is noteworthy that the
restart cycle length $\plen$ plays an important role of balancing the tradeoffs
that $1)$ the optimal policies between two adjacent restarts are relatively
stationary among themselves so that the algorithm is compared to stable
benchmarks, and that $2)$ there are not too many restarts so that the
sub-optimality of algorithm do not grow too fast when combined over periods
in between restarts.

\paragraph{Comparison with existing results.}

We compare the results in Theorems \ref{thm:regret_dyn_oppo} and \ref{alg:acc_oppo} to those in \cite{cheung2019reinforcement}, which is so far state-of-the-art in dynamic regret analysis for non-stationary RL.
First, our model-free algorithms are more efficient than the model-based algorithm in \cite{cheung2019reinforcement} that is adapted from UCRL2 and requires solving linear programs in each timestep. 
Second, our bounds in Theorems \ref{thm:regret_dyn_oppo} and \ref{thm:regret_acc_oppo} are on the near-optimal order $\tilde{O}(T^{1/2})$ when $\polchange$ is sufficiently small, whereas the results in \citet{cheung2019reinforcement} are of order $\tilde{O}(T^{2/3})$.
On the other hand, \cite{cheung2019reinforcement} studies a more general setting where the transition kernel of the MDP is allowed to vary adversarially in each timestep. It also provides a procedure to adaptively tune its UCRL2-based algorithm to achieve an $\tilde{O}(T^{3/4})$ regret without knowledge of variations such as $\polchange$.

\section*{Acknowledgement}

This work is supported in part by  National Science Foundation Grant CCF-1704828.

\nocite{}
\bibliographystyle{plainnat}
\bibliography{references}

\newpage

\appendix
\appendixpage

\section{Implementation of EvaluatePolicy} \label{sec:implem_evalpolicy}
\begin{algorithm}
	\begin{algorithmic}[1]
		
		\Require Episode index $k$, reward functions $\{r_{h}\}$, policies
		$\{\pi_{h}\}$, regularization factor $\lambda$ and bonus multiplier $\beta$
		
		\Ensure Updated Q-values $\{Q_{h}\}$
		
		\State Initialize $V_{H+1}$ as a zero function
		
		\For{ step $h=H,H-1,\ldots,1$}
		
		\State $\Lambda_{h}\leftarrow\sum_{t\in[k-1]}\fmap(s_{h}^{t},a_{h}^{t})\fmap(s_{h}^{t},a_{h}^{t})^{\top}+\lambda\cdot\IdMat$
		\label{line:ep_cov_update}
		
		\State $w_{h}\leftarrow(\Lambda_{h})^{-1}\sum_{t\in[k-1]}\fmap(s_{h}^{t},a_{h}^{t})\cdot V_{h+1}(s_{h+1}^{t})$\label{line:ep_dyn_oppo_w_def}
		
		\State $\Gamma_{h}(\cdot,\cdot)\leftarrow\beta\cdot[\fmap(\cdot,\cdot)^{\top}(\Lambda_{h})^{-1}\fmap(\cdot,\cdot)]^{1/2}$
		\label{line:ep_bonus_def}
		
		\State $Q_{h}(\cdot,\cdot)\leftarrow r_{h}(\cdot,\cdot)+\min\{\fmap(\cdot,\cdot)^{\top}w_{h}+\Gamma_{h}(\cdot,\cdot),H-h\}^{+}$
		\label{line:ep_Q_def}
		
		\State $V_{h}(\cdot)\leftarrow\left\langle Q_{h}(\cdot,\cdot),\pi_{h}(\cdot\ |\ \cdot)\right\rangle _{\cA}$
		
		\EndFor
		
	\end{algorithmic}
	
	\caption{EvaluatePolicy \label{alg:eval_pol}}
\end{algorithm}

In Algorithm \ref{alg:eval_pol}, the tuples $\{(s^{t}_{h}, a^{t}_{h})\}_{t\in[k-1]}$ are state-action pairs visited by Algorithm \ref{alg:dyn_oppo} or \ref{alg:acc_oppo} before episode $k$.

\section{Proofs of technical lemmas}

Recall that $\pnum\coloneqq\left\lceil \frac{K}{\plen}\right\rceil $.  Algorithm \ref{alg:dyn_oppo} divides $K$ episodes into
$\pnum$ periods, and at the the beginning of each period it resets its Q-value and policy estimates. Each period contains $\plen$ episodes, except for
the last one, which consists of at most $\plen$ episodes. For ease
of notations, we assume that the last period has exactly $\plen$
episodes. Our proof can be easily extended to the case where the last
period has fewer than $\plen$ episodes.

\subsection{Regret decomposition}

For any $(k,h,s)\in[K]\times[H]\times\cS$, we define the model prediction
error 
\begin{equation}
\iota_{h}^{k}\coloneqq r_{h}^{k}+\P_{h}V_{h+1}^{k}-Q_{h}^{k}.\label{eq:model_pred_err}
\end{equation}
We have the following decomposition of the dynamic regret \eqref{eq:dynamic_regret}.
\begin{lem}
	\emph{\label{lem:dyn_reg_decomp}  We have }
	\begin{align*}
	\dreg(K) & =\sum_{l\in[\pnum]}\sum_{k=(l-1)\plen+1}^{l\plen}\sum_{h\in[H]}\E_{\pi^{*,k}}\left[\left\langle Q_{h}^{k}(s_{h},\cdot),\pi_{h}^{*,k}(\cdot\ |\ s_{h})-\pi_{h}^{k}(\cdot\ |\ s_{h})\right\rangle \ \Bigg|\ s_{1}=s_{1}^{k}\right]\\
	& \quad+\sum_{l\in[\pnum]}\sum_{k=(l-1)\plen+1}^{l\plen}\sum_{h\in[H]}\left[\E_{\pi^{*,k}}[\iota_{h}^{k}(s_{h},a_{h})\ |\ s_{1}=s_{1}^{k}]-\iota_{h}^{k}(s_{h}^{k},a_{h}^{k})\right]+M_{K,H},
	\end{align*}
	where $M_{K,H}\coloneqq\sum_{k\in[K]}\sum_{h\in[H]}M_{h}^{k}$ is
	a martingale that satisfies $\left|M_{h}^{k}\right|\le4H$ for $(k,h)\in[K]\times[H]$.
\end{lem}
We defer its proof to Section \ref{sec:proof_dyn_reg_decomp}.

\subsection{Performance difference bound}

We may further decompose the first term on the RHS of Lemma \ref{lem:dyn_reg_decomp}
as 
\begin{align}
& \quad\sum_{l\in[\pnum]}\sum_{k=(l-1)\plen+1}^{l\plen}\sum_{h\in[H]}\E_{\pi^{*,k}}\left[\left\langle Q_{h}^{k}(s_{h},\cdot),\pi_{h}^{*,k}(\cdot\ |\ s_{h})-\pi_{h}^{k}(\cdot\ |\ s_{h})\right\rangle \ \Bigg|\ s_{1}=s_{1}^{k}\right]\nonumber \\
& =\sum_{l\in[\pnum]}\sum_{k=(l-1)\plen+1}^{l\plen}\sum_{h\in[H]}\E_{\pi^{*,(l-1)\plen+1}}\left[\left\langle Q_{h}^{k}(s_{h},\cdot),\pi_{h}^{*,k}(\cdot\ |\ s_{h})-\pi_{h}^{k}(\cdot\ |\ s_{h})\right\rangle \ \Bigg|\ s_{1}=s_{1}^{k}\right]\nonumber \\
& \quad+\sum_{l\in[\pnum]}\sum_{k=(l-1)\plen+1}^{l\plen}\sum_{h\in[H]}\left(\E_{\pi^{*,k}}-\E_{\pi^{*,(l-1)\plen+1}}\right)\left[\left\langle Q_{h}^{k}(s_{h},\cdot),\pi_{h}^{*,k}(\cdot\ |\ s_{h})-\pi_{h}^{k}(\cdot\ |\ s_{h})\right\rangle \ \Bigg|\ s_{1}=s_{1}^{k}\right].\label{eq:perf_diff_decomp}
\end{align}

\subsubsection{First term in Equation \eqref{eq:perf_diff_decomp}}

We first introduce a ``one-step descent'' result.
\begin{lem}[{\citet[Lemma 3.3]{cai2019provably}}]
	\label{lem:one_step_descent}\emph{ } For any distribution $p^{*}$
	and $p$ supported on $\cA$, state $s\in\cS$, and function $Q:\cS\times\cA\to[0,H]$,
	it holds for a distribution $p'$ supported on $\cA$ with $p'(\cdot)\propto p(\cdot)\cdot\exp\{\alpha\cdot Q(s,\cdot)\}$
	that 
	\[
	\left\langle Q(s,\cdot),p^{*}(\cdot)-p(\cdot)\right\rangle \le\frac{1}{2}\alpha H^{2}+\frac{1}{\alpha}\left[\kl(p^{*}(\cdot)\ \|\ p(\cdot))-\kl(p^{*}(\cdot)\ \|\ p'(\cdot))\right].
	\]
\end{lem}
The next lemma controls the performance difference for any initial
state.
\begin{lem}
	\emph{\label{lem:perf_diff_bound_fixed_expect} }For any $s_{1}^{k}\in\cS$,
	we have
	\begin{align*}
	& \quad\sum_{l\in[\pnum]}\sum_{k=(l-1)\plen+1}^{l\plen}\sum_{h\in[H]}\E_{\pi^{*,(l-1)\plen+1}}\left[\left\langle Q_{h}^{k}(s_{h},\cdot),\pi_{h}^{*,k}(\cdot\ |\ s_{h})-\pi_{h}^{k}(\cdot\ |\ s_{h})\right\rangle \ \Bigg|\ s_{1}=s_{1}^{k}\right]\\
	& \le\frac{1}{2}\alpha KH^{3}+\frac{1}{\alpha}LH\log A+\plen H\polchange.
	\end{align*}
\end{lem}
\begin{proof}
	For each $l\in[\pnum]$, we let $\nu^{l}=\{\nu_{h}^{l}\}_{h\in[H]}$
	where each $\nu_{h}^{l}$ is a policy (or a distribution supported
	$\cA$) to be specified. We have the decomposition 
	\begin{align}
	& \quad\sum_{l\in[\pnum]}\sum_{k=(l-1)\plen+1}^{l\plen}\sum_{h\in[H]}\E_{\pi^{*,(l-1)\plen+1}}\left[\left\langle Q_{h}^{k}(s_{h},\cdot),\pi_{h}^{*,k}(\cdot\ |\ s_{h})-\pi_{h}^{k}(\cdot\ |\ s_{h})\right\rangle \ \Bigg|\ s_{1}=s_{1}^{k}\right]\nonumber \\
	& =\sum_{l\in[\pnum]}\sum_{k=(l-1)\plen+1}^{l\plen}\sum_{h\in[H]}\E_{\pi^{*,(l-1)\plen+1}}\left[\left\langle Q_{h}^{k}(s_{h},\cdot),\nu_{h}^{l}(\cdot\ |\ s_{h})-\pi_{h}^{k}(\cdot\ |\ s_{h})\right\rangle \ \Bigg|\ s_{1}=s_{1}^{k}\right]\nonumber \\
	& \quad+\sum_{l\in[\pnum]}\sum_{k=(l-1)\plen+1}^{l\plen}\sum_{h\in[H]}\E_{\pi^{*,(l-1)\plen+1}}\left[\left\langle Q_{h}^{k}(s_{h},\cdot),\pi_{h}^{*,k}(\cdot\ |\ s_{h})-\nu_{h}^{l}(\cdot\ |\ s_{h})\right\rangle \ \Bigg|\ s_{1}=s_{1}^{k}\right]\nonumber \\
	& \eqqcolon E_{1}+E_{2}.\label{eq:dyn_oppo_D1_D2_decomp}
	\end{align}
	By Lemma \ref{lem:one_step_descent}, we have 
	\begin{align*}
	E_{1} & \le\frac{1}{2}\alpha KH^{3}+\sum_{h\in[H]}\frac{1}{\alpha}\\
	& \quad\times\sum_{l\in[\pnum]}\E_{\pi^{*,(l-1)\plen+1}}\left[\sum_{k=(l-1)\plen+1}^{l\plen}\left[\kl(\nu_{h}^{l}(\cdot\ |\ s_{h})\ \|\ \pi_{h}^{k}(\cdot\ |\ s_{h}))-\kl(\nu_{h}^{l}(\cdot\ |\ s_{h})\ \|\ \pi_{h}^{k+1}(\cdot\ |\ s_{h}))\right]\ \Bigg|\ s_{1}=s_{1}^{k}\right]\\
	& \le\frac{1}{2}\alpha KH^{3}+\sum_{h\in[H]}\frac{1}{\alpha}\\
	& \quad\times\sum_{l\in[\pnum]}\E_{\pi^{*,(l-1)\plen+1}}\left[\kl(\nu_{h}^{l}(\cdot\ |\ s_{h})\ \|\ \pi_{h}^{(l-1)\plen+1}(\cdot\ |\ s_{h}))-\kl(\nu_{h}^{l}(\cdot\ |\ s_{h})\ \|\ \pi_{h}^{l\plen+1}(\cdot\ |\ s_{h}))\ \Bigg|\ s_{1}=s_{1}^{k}\right]\\
	& \le\frac{1}{2}\alpha KH^{3}+\sum_{h\in[H]}\frac{1}{\alpha}\cdot\sum_{l\in[\pnum]}\E_{\pi^{*,(l-1)\plen+1}}\left[\kl(\nu_{h}^{l}(\cdot\ |\ s_{h})\ \|\ \pi_{h}^{(l-1)\plen+1}(\cdot\ |\ s_{h}))\ \Bigg|\ s_{1}=s_{1}^{k}\right]\\
	& \le\frac{1}{2}\alpha KH^{3}+\frac{1}{\alpha}LH\log A,
	\end{align*}
	where the second step holds by telescoping, the third step holds 
	since the KL divergence is non-negative, and the last step holds since
	by construction $\pi_{h}^{(l-1)\plen+1}(\cdot\ |\ s)$ in Algorithm
	\ref{alg:dyn_oppo} is a uniform distribution on $\cA$ and for any
	policy $\nu$ and state $s\in\cS$ we have 
	\begin{align*}
	\kl(\nu(\cdot\ |\ s)\|\pi_{h}^{(l-1)\plen+1}(\cdot\ |\ s)) & =\sum_{a\in\cA}\nu(a\ |\ s)\cdot\log\left(A\cdot\nu(a\ |\ s)\right)\\
	& =\log A+\sum_{a\in\cA}\nu(a\ |\ s)\cdot\log\left(\nu(a\ |\ s)\right)\\
	& \le\log A
	\end{align*}
	given the fact that the entropy of any distribution is non-negative. 
	
	Now for each $(l,h)\in[\pnum]\times[H]$, we set 
	\[
	\nu_{h}^{l}\coloneqq\pi_{h}^{*,(l-1)\plen+1},
	\]
	that is, $\nu_{h}^{l}$ is the policy after one update in step $h$
	of period $l$. For $D_{2}$, we have 
	\begin{align*}
	E_{2} & \le\sum_{l\in[\pnum]}\sum_{k=(l-1)\plen+1}^{l\plen}\sum_{h\in[H]}\E_{\pi^{*,(l-1)\plen+1}}\left[H\cdot\norm[\pi_{h}^{*,k}(\cdot\ |\ s_{h})-\nu_{h}^{l}(\cdot\ |\ s_{h})]1\ \Bigg|\ s_{1}=s_{1}^{k}\right]\\
	& =H\cdot\sum_{l\in[\pnum]}\sum_{k=(l-1)\plen+1}^{l\plen}\sum_{h\in[H]}\E_{\pi^{*,(l-1)\plen+1}}\left[\norm[\pi_{h}^{*,k}(\cdot\ |\ s_{h})-\pi_{h}^{*,(l-1)\plen+1}(\cdot\ |\ s_{h})]1\ \Bigg|\ s_{1}=s_{1}^{k}\right]\\
	& \le H\cdot\sum_{l\in[\pnum]}\sum_{k=(l-1)\plen+1}^{l\plen}\sum_{h\in[H]}\sum_{t=(l-1)\plen+2}^{k}\E_{\pi^{*,(l-1)\plen+1}}\left[\norm[\pi_{h}^{*,t}(\cdot\ |\ s_{h})-\pi_{h}^{*,t-1}(\cdot\ |\ s_{h})]1\ \Bigg|\ s_{1}=s_{1}^{k}\right]\\
	& \le H\cdot\sum_{l\in[\pnum]}\sum_{k=(l-1)\plen+1}^{l\plen}\sum_{t=(l-1)\plen+1}^{l\plen}\sum_{h\in[H]}\max_{s'\in\cS}\norm[\pi_{h}^{*,t}(\cdot\ |\ s')-\pi_{h}^{*,t-1}(\cdot\ |\ s')]1\\
	& =H\cdot\plen\cdot\sum_{t\in[K]}\sum_{h\in[H]}\max_{s'\in\cS}\norm[\pi_{h}^{*,t}(\cdot\ |\ s')-\pi_{h}^{*,t-1}(\cdot\ |\ s')]1\\
	& =H\cdot\plen\cdot\polchange
	\end{align*}
	where the first step holds by Holder's inequality and the fact that
	$\norm[Q_{h}^{k}(s,\cdot)]{\infty}\le H$, the second step holds by
	the definition of $\{\nu_{h}^{l}\}$, the third step follows from
	telescoping, and the last step holds by the definition $\polchange\coloneqq\sum_{k\in[K]}\sum_{h\in[H]}\norm[\pi_{h}^{*,k}-\pi{}_{h}^{*,k-1}]{\infty}$.
\end{proof}

\subsubsection{Second term in Equation \eqref{eq:perf_diff_decomp}}

The following lemma controls the performance difference due to varying
optimal policies across episodes.
\begin{lem}
	\emph{\label{lem:perf_diff_bound_varying_policies} }Under Assumption
	\ref{asp:visit_measure_smooth}, we have 
	\begin{align*}
	& \quad\sum_{l\in[\pnum]}\sum_{k=(l-1)\plen+1}^{l\plen}\sum_{h\in[H]}\left(\E_{\pi^{*,k}}-\E_{\pi^{*,(l-1)\plen+1}}\right)\left[\left\langle Q_{h}^{k}(s_{h},\cdot),\pi_{h}^{*,k}(\cdot\ |\ s_{h})-\pi_{h}^{k}(\cdot\ |\ s_{h})\right\rangle \ \Bigg|\ s_{1}=s_{1}^{k}\right]\\
	& \le C\cdot\plen H^{2}\polchange,
	\end{align*}
	where $C>0$ is a universal constant.
\end{lem}
\begin{proof}
	We denote by $\indic(s_{h})$ the indicator function for state $s_{h}$,
	and we have 
	\begin{align}
	& \quad\sum_{l\in[\pnum]}\sum_{k=(l-1)\plen+1}^{l\plen}\sum_{h\in[H]}\left(\E_{\pi^{*,k}}-\E_{\pi^{*,(l-1)\plen+1}}\right)\left[\left\langle Q_{h}^{k}(s_{h},\cdot),\pi_{h}^{*,k}(\cdot\ |\ s_{h})-\pi_{h}^{k}(\cdot\ |\ s_{h})\right\rangle \ \Bigg|\ s_{1}=s_{1}^{k}\right]\nonumber \\
	& \le\sum_{l\in[\pnum]}\sum_{k=(l-1)\plen+1}^{l\plen}\sum_{h\in[H]}\left(\E_{\pi^{*,k}}-\E_{\pi^{*,(l-1)\plen+1}}\right)\left[2H\cdot\indic(s_{h})\ \Big|\ s_{1}=s_{1}^{k}\right]\nonumber \\
	& =\sum_{l\in[\pnum]}\sum_{k=(l-1)\plen+1}^{l\plen}\sum_{h\in[H]}\sum_{t=(l-1)\plen+2}^{k}\left(\E_{\pi^{*,t}}-\E_{\pi^{*,t-1}}\right)\left[2H\cdot\indic(s_{h})\ \Big|\ s_{1}=s_{1}^{k}\right]\label{eq:perf_diff_J2_telescoped}
	\end{align}
	where the first step follows from $\left|\left\langle Q_{h}^{k}(s_{h},\cdot),\pi_{h}^{*,k}(\cdot\ |\ s_{h})-\pi_{h}^{k}(\cdot\ |\ s_{h})\right\rangle \right|\le2H\cdot\indic(s_{h})$
	and the last step holds by telescoping. Let $\cP_{i}^{\pi}(s)$ be
	the visitation measure of state $s$ in step $i$ under policy $\pi$,
	and let us fix an $h\in[H]$. Under policies $\{\pi^{(i)}\}$, the
	distribution of $s_{h}$ conditional on $s_{1}$ is given by 
	\begin{align*}
	\cP_{1}^{\pi^{(1)}}\cP_{2}^{\pi^{(2)}}\cdots\cP_{h-1}^{\pi^{(h-1)}}(s_{h}\ |\ s_{1}) & \coloneqq\sum_{s_{2},\ldots,s_{h-1}}\prod_{i\in[h-1]}\cP_{i}^{\pi^{(i)}}(s_{i+1}\ |\ s_{i}).
	\end{align*}
	Recall that $\norm[\pi-\pi']{\infty}\coloneqq\max_{s\in\cS}\norm[\pi(\cdot\ |\ s)-\pi'(\cdot\ |\ s)]1$
	for any pair of policies $\pi$ and $\pi'$, and $\cP_{h}^{\pi}(s\ |\ s')\coloneqq\sum_{a'\in\cA}\cP_{h}(s\ |\ s',a')\cdot\pi_{h}(a'\ |\ s')$
	is the transition kernel in step $h$ when policy $\pi$ is executed.
	We have the following smoothness property for the (conditional) visitation
	measure $\cP_{1}^{\pi^{(1)}}\cP_{2}^{\pi^{(2)}}\cdots\cP_{h-1}^{\pi^{(h-1)}}(s_{h}\ |\ s_{1})$
	thanks to Assumption \ref{asp:visit_measure_smooth}.
	\begin{lem}
		\label{lem:multi_visit_measure_smooth}Under Assumption \ref{asp:visit_measure_smooth},
		for any $h\in[H]$, $j\in[h-1]$, $s_{h},s_{1}\in\cS$, and policies
		$\{\pi^{(i)}\}_{i\in[H]}\cup\{\pi'\}$ we have 
		\begin{align*}
		& \quad\left|\cP_{1}^{\pi^{(1)}}\cdots\cP_{j}^{\pi^{(j)}}\cdots\cP_{h-1}^{\pi^{(h-1)}}(s_{h}\ |\ s_{1})-\cP_{1}^{\pi^{(1)}}\cdots\cP_{j}^{\pi'}\cdots\cP_{h-1}^{\pi^{(h-1)}}(s_{h}\ |\ s_{1})\right|\\
		& \le C\cdot\norm[\pi_{j}^{(j)}-\pi_{j}']{\infty},
		\end{align*}
		where $C>0$ is a universal constant.
	\end{lem}
	\begin{proof}
		We have 
		\begin{align*}
		& \quad\left|\cP_{1}^{\pi^{(1)}}\cdots\cP_{j}^{\pi^{(j)}}\cdots\cP_{h-1}^{\pi^{(h-1)}}(s_{h}\ |\ s_{1})-\cP_{1}^{\pi^{(1)}}\cdots\cP_{j}^{\pi'}\cdots\cP_{h-1}^{\pi^{(h-1)}}(s_{h}\ |\ s_{1})\right|\\
		& \le\sum_{s_{2},s_{3},\ldots,s_{h-1}}\left|\cP_{j}^{\pi^{(j)}}(s_{j+1}\ |\ s_{j})-\cP_{j}^{\pi'}(s_{j+1}\ |\ s_{j})\right|\cdot\prod_{i\in[h-1]\backslash\{j\}}\cP_{i}^{\pi^{(i)}}(s_{i+1}\ |\ s_{i})\\
		& \overset{(i)}{\le}\sum_{s_{2},\ldots s_{j},s_{j+2},\ldots,s_{h-1}}\sum_{s_{j+1}}\left|\cP_{j}^{\pi^{(j)}}(s_{j+1}\ |\ s_{j})-\cP_{j}^{\pi'}(s_{j+1}\ |\ s_{j})\right|\cdot\max_{s_{j+1}\in\cS}\prod_{i\in[h-1]\backslash\{j\}}\cP_{i}^{\pi^{(i)}}(s_{i+1}\ |\ s_{i})\\
		& \overset{(ii)}{\le}\sum_{s_{2},\ldots s_{j-1},s_{j+2},\ldots,s_{h-1}}\max_{s_{j}\in\cS}\sum_{s_{j+1}}\left|\cP_{j}^{\pi^{(j)}}(s_{j+1}\ |\ s_{j})-\cP_{j}^{\pi'}(s_{j+1}\ |\ s_{j})\right|\cdot\sum_{s_{j}}\max_{s_{j+1}\in\cS}\prod_{i\in[h-1]\backslash\{j\}}\cP_{i}^{\pi^{(i)}}(s_{i+1}\ |\ s_{i})\\
		& \overset{(iii)}{\le}C\cdot\norm[\pi_{j}^{(j)}-\pi_{j}']{\infty}\cdot\sum_{s_{2},\ldots,s_{j},s_{j+2}\ldots,s_{h-1}}\max_{s_{j+1}\in\cS}\prod_{i\in[h-1]\backslash\{j\}}\cP_{i}^{\pi^{(i)}}(s_{i+1}\ |\ s_{i})\\
		& =C\cdot\norm[\pi_{j}^{(j)}-\pi_{j}']{\infty}\cdot\underbrace{\sum_{s_{j+2},\ldots,s_{h-1}}\max_{s_{j+1}\in\cS}\prod_{i=j+1}^{h-1}\cP_{i}^{\pi^{(i)}}(s_{i+1}\ |\ s_{i})}_{\le1}\cdot\underbrace{\sum_{s_{2},\ldots,s_{j}}\prod_{i=1}^{j-1}\cP_{i}^{\pi^{(i)}}(s_{i+1}\ |\ s_{i})}_{=1}\\
		& \le C\cdot\norm[\pi_{j}^{(j)}-\pi_{j}']{\infty},
		\end{align*}
		where steps $(i)$ and $(ii)$ hold by Holder's inequality, and step
		$(iii)$ holds under Assumption \ref{asp:visit_measure_smooth}.
	\end{proof}
	Therefore, for $(k,t,h)\in[K]^{2}\times[H]$ such that $k\le t-1$,
	we have 
	\begin{align}
	& \quad\left|\left(\E_{\pi^{*,t}}-\E_{\pi^{*,t-1}}\right)\left[\indic(s_{h})\ \Big|\ s_{1}=s_{1}^{k}\right]\right|\nonumber \\
	& \le\norm[\cP_{1}^{\pi^{*,t}}\cP_{2}^{\pi^{*,t}}\cdots\cP_{h-1}^{\pi^{*,t}}(\cdot\ |\ s_{1}^{k})-\cP_{1}^{\pi^{*,t-1}}\cP_{2}^{\pi^{*,t-1}}\cdots\cP_{h-1}^{\pi^{*,t-1}}(\cdot\ |\ s_{1}^{k})]{\infty}\nonumber \\
	& \le\norm[\cP_{1}^{\pi^{*,t}}\cP_{2}^{\pi^{*,t}}\cdots\cP_{h-1}^{\pi^{*,t}}(\cdot\ |\ s_{1}^{k})-\cP_{1}^{\pi^{*,t}}\cP_{2}^{\pi^{*,t-1}}\cdots\cP_{h-1}^{\pi^{*,t-1}}(\cdot\ |\ s_{1}^{k})]{\infty}\nonumber \\
	& \quad+\norm[\cP_{1}^{\pi^{*,t}}\cP_{2}^{\pi^{*,t-1}}\cdots\cP_{h-1}^{\pi^{*,t-1}}(\cdot\ |\ s_{1}^{k})-\cP_{1}^{\pi^{*,t-1}}\cP_{2}^{\pi^{*,t-1}}\cdots\cP_{h-1}^{\pi^{*,t-1}}(\cdot\ |\ s_{1}^{k})]{\infty}\nonumber \\
	& \le C\cdot\sum_{i\in[h]}\norm[\pi_{i}^{*,t}-\pi_{i}^{*,t-1}]{\infty},\label{eq:perf_diff_J2_visit_meas_bound}
	\end{align}
	where the third step follows from further telescoping the first term
	in the second step and then applying Lemma \ref{lem:multi_visit_measure_smooth}.
	Combining Equations \eqref{eq:perf_diff_J2_telescoped} and \eqref{eq:perf_diff_J2_visit_meas_bound},
	we have 
	\begin{align*}
	& \quad\sum_{l\in[\pnum]}\sum_{k=(l-1)\plen+1}^{l\plen}\sum_{h\in[H]}\left(\E_{\pi^{*,k}}-\E_{\pi^{*,(l-1)\plen+1}}\right)\left[\left\langle Q_{h}^{k}(s_{h},\cdot),\pi_{h}^{*,k}(\cdot\ |\ s_{h})-\pi_{h}^{k}(\cdot\ |\ s_{h})\right\rangle \ \Bigg|\ s_{1}=s_{1}^{k}\right]\\
	& \le\sum_{l\in[\pnum]}\sum_{k=(l-1)\plen+1}^{l\plen}\sum_{h\in[H]}\sum_{t=(l-1)\plen+2}^{k}2H\cdot C\cdot\sum_{i\in[h]}\norm[\pi_{i}^{*,t}-\pi_{i}^{*,t-1}]{\infty}\\
	& \le2H\cdot C\cdot\sum_{h\in[H]}\left(\sum_{l\in[\pnum]}\sum_{k=(l-1)\plen+1}^{l\plen}\sum_{t=(l-1)\plen+1}^{l\plen}\sum_{i\in[H]}\norm[\pi_{i}^{*,t}-\pi_{i}^{*,t-1}]{\infty}\right)\\
	& =2H\cdot C\cdot\sum_{h\in[H]}\left(\plen\sum_{t\in[K]}\sum_{i\in[H]}\norm[\pi_{i}^{*,t}-\pi_{i}^{*,t-1}]{\infty}\right)\\
	& \le2C\cdot H^{2}\cdot\plen\cdot\polchange,
	\end{align*}
	where in the last step we used the definition $\polchange\coloneqq\sum_{k\in[K]}\sum_{i\in[H]}\norm[\pi_{i}^{*,k}-\pi_{i}^{*,k-1}]{\infty}$.
\end{proof}

\subsubsection{Putting together}

Finally, we establish the following result on the performance difference.
\begin{lem}
	\label{lem:perf_diff_bound_visit_meas}Recall that $\polchange\coloneqq\sum_{k\in[K]}\sum_{i\in[H]}\norm[\pi_{i}^{*,k}-\pi_{i}^{*,k-1}]{\infty}$.
	Under Assumption \ref{asp:visit_measure_smooth}, we choose $\alpha=\sqrt{\frac{L\log A}{KH^{2}}}$
	in Algorithm \ref{alg:dyn_oppo}, and we have 
	\begin{align*}
	& \quad\sum_{l\in[\pnum]}\sum_{k=(l-1)\plen+1}^{l\plen}\sum_{h\in[H]}\E_{\pi^{*,k}}\left[\left\langle Q_{h}^{k}(s_{h},\cdot),\pi_{h}^{*,k}(\cdot\ |\ s_{h})-\pi_{h}^{k}(\cdot\ |\ s_{h})\right\rangle \ \Bigg|\ s_{1}=s_{1}^{k}\right]\\
	& =2H^{2}\sqrt{K\log A}+C\cdot\plen H^{2}\polchange,
	\end{align*}
	for some universal constant $C>0$.
\end{lem}
\begin{proof}
	Recall from Equation \eqref{eq:perf_diff_decomp} that for any $l\in[\pnum]$,
	we have 
	\begin{align*}
	& \quad\sum_{l\in[\pnum]}\sum_{k=(l-1)\plen+1}^{l\plen}\sum_{h\in[H]}\E_{\pi^{*,k}}\left[\left\langle Q_{h}^{k}(s_{h},\cdot),\pi_{h}^{*,k}(\cdot\ |\ s_{h})-\pi_{h}^{k}(\cdot\ |\ s_{h})\right\rangle \ \Bigg|\ s_{1}=s_{1}^{k}\right]\\
	& =\sum_{l\in[\pnum]}\sum_{k=(l-1)\plen+1}^{l\plen}\sum_{h\in[H]}\E_{\pi^{*,(l-1)\plen+1}}\left[\left\langle Q_{h}^{k}(s_{h},\cdot),\pi_{h}^{*,k}(\cdot\ |\ s_{h})-\pi_{h}^{k}(\cdot\ |\ s_{h})\right\rangle \ \Bigg|\ s_{1}=s_{1}^{k}\right]\\
	& \quad+\sum_{l\in[\pnum]}\sum_{k=(l-1)\plen+1}^{l\plen}\sum_{h\in[H]}\left(\E_{\pi^{*,k}}-\E_{\pi^{*,(l-1)\plen+1}}\right)\left[\left\langle Q_{h}^{k}(s_{h},\cdot),\pi_{h}^{*,k}(\cdot\ |\ s_{h})-\pi_{h}^{k}(\cdot\ |\ s_{h})\right\rangle \ \Bigg|\ s_{1}=s_{1}^{k}\right].
	\end{align*}
	By applying Lemmas \ref{lem:perf_diff_bound_fixed_expect} and \ref{lem:perf_diff_bound_varying_policies},
	we have 
	\begin{align*}
	& \quad\sum_{l\in[\pnum]}\sum_{k=(l-1)\plen+1}^{l\plen}\sum_{h\in[H]}\E_{\pi^{*,k}}\left[\left\langle Q_{h}^{k}(s_{h},\cdot),\pi_{h}^{*,k}(\cdot\ |\ s_{h})-\pi_{h}^{k}(\cdot\ |\ s_{h})\right\rangle \ \Bigg|\ s_{1}=s_{1}^{k}\right]\\
	& \le\alpha KH^{3}+\frac{1}{\alpha}LH\log A+\plen H\polchange+C'\cdot\plen H^{2}\polchange\\
	& =2H^{2}\sqrt{KL\log A}+\plen H\polchange+C'\cdot\plen H^{2}\polchange\\
	& \le2H^{2}\sqrt{KL\log A}+C\cdot\plen H^{2}\polchange,
	\end{align*}
	where the equality above holds by our choice of $\alpha$, and $C,C'>0$
	are universal constants.
\end{proof}

\subsection{Model prediction error }

We need the following results to control the bonus $\Gamma_{h}^{k}(\cdot,\cdot)$
(defined in Line \ref{line:ep_bonus_def} of Algorithm \ref{alg:eval_pol})
accumulated over episodes.
\begin{lem}
	\emph{\label{lem:bonus_ucb} }Let $\lambda=1$ and $\beta=C\cdot H\sqrt{S\log(dT/p)}$
	in Algorithm \ref{alg:dyn_oppo}, where $C>0$ is a universal constant
	and $p\in(0,1]$. With probability at least $1-p/2$ and for all $(k,h,s,a)\in[K]\times[H]\times\cS\times\cA$,
	it holds that 
	\[
	-2\Gamma_{h}^{k}(s,a)\le\iota_{h}^{k}(s,a)\le0.
	\]
	
\end{lem}
\begin{proof}
	The proof follows that of \citet[Lemma 4.3]{cai2019provably} specialized
	to the tabular setting by replacing Lemma D.2 therein with \citet[Lemma 12]{bai2020provable}.
\end{proof}
\begin{lem}[{\citet[Lemma D.6]{cai2019provably}; \citet[Lemma D.2]{jin2019provably}}]
	\emph{\label{lem:lsvi_quad_Lambda_pilot_bound}  }Let $\{\fmap_{t}\}_{t\ge0}$
	be a bounded sequence in $\real^{d}$ satisfying $\sup_{t\ge0}\norm[\fmap_{t}]{}\le1$.
	Let $\Lambda_{0}\in\real^{d\times d}$ be a positive definite matrix
	with $\lambda_{\min}(\Lambda_{0})\ge1$. For any $t\ge0$, we define
	$\Lambda_{t}\coloneqq\Lambda_{0}+\sum_{i\in[t-1]}\fmap_{i}\fmap_{i}^{\top}$.
	Then, we have 
	\[
	\log\left[\frac{\det(\Lambda_{t+1})}{\det(\Lambda_{0})}\right]\le\sum_{i\in[t]}\fmap_{i}^{\top}\Lambda_{i}^{-1}\fmap_{i}\le2\log\left[\frac{\det(\Lambda_{t+1})}{\det(\Lambda_{0})}\right].
	\]
\end{lem}
\begin{lem}
	\label{lem:bonus_sum_bound}We have 
	\[
	\sum_{k\in[K[}\sum_{h\in[H]}\Gamma_{h}^{k}(s_{h}^{k},a_{h}^{k})\le\beta H\sqrt{2dK\log((K+\lambda)/\lambda)}.
	\]
\end{lem}
\begin{proof}
	Given the construction of $\Lambda_{h}^{k}$ in Algorithm \ref{alg:dyn_oppo},
	we have for any $h\in[H]$, 
	\begin{align*}
	\sum_{k\in[K]}\fmap(s_{h}^{k},a_{h}^{k})^{\top}(\Lambda_{h}^{k})^{-1}\fmap(s_{h}^{k},a_{h}^{k}) & \le2\log\left[\frac{\det(\Lambda_{h}^{K+1})}{\det(\Lambda_{h}^{1})}\right]\\
	& \le2d\log\left[\frac{K+\lambda}{\lambda}\right],
	\end{align*}
	where the last step holds since the construction of Algorithm \ref{alg:dyn_oppo}
	implies that $\Lambda_{h}^{1}=\lambda\cdot\IdMat$ and 
	\[
	\Lambda_{h}^{k+1}=\sum_{t\in[k]}\fmap(s_{h}^{t},a_{h}^{t})\fmap(s_{h}^{t},a_{h}^{t})^{\top}+\lambda\cdot\IdMat\preceq(k+\lambda)\cdot\IdMat,
	\]
	which yields 
	\[
	\log\left[\frac{\det(\Lambda_{h}^{K+1})}{\det(\Lambda_{h}^{1})}\right]\le\log\left[\frac{\det((K+\lambda)\cdot\IdMat)}{\det(\lambda\cdot\IdMat)}\right]=d\log\left[\frac{K+\lambda}{\lambda}\right].
	\]
	Therefore, by the Cauchy-Schwarz inequality and Lemma \ref{lem:bonus_sum_bound},
	we have 
	\begin{align*}
	\sum_{k\in[K[}\sum_{h\in[H]}\Gamma_{h}^{k}(s_{h}^{k},a_{h}^{k}) & \le\beta\cdot\sum_{h\in[H]}\left(K\cdot\sum_{k\in[K]}\fmap(s_{h}^{k},a_{h}^{k})^{\top}(\Lambda_{h}^{k})^{-1}\fmap(s_{h}^{k},a_{h}^{k})\right)^{1/2}\\
	& =\beta H\sqrt{2dK\log((K+\lambda)/\lambda)}.
	\end{align*}
\end{proof}

\subsection{Martingale bound}
\begin{lem}
	\label{lem:mtg_bound}Consider $M_{K,H}$ in Lemma \ref{lem:dyn_reg_decomp}.
	With probability $1-\delta/2$, we have 
	\[
	\left|M_{K,H}\right|\le\sqrt{16H^{2}T\cdot\log(4/\delta)}.
	\]
\end{lem}
\begin{proof}
	From Lemma \ref{lem:dyn_reg_decomp} and by the Azuma Hoeffding inequality,
	we have for any $t\ge0$, 
	\[
	\P\left(\left|M_{K,H}\right|\ge t\right)\le2\exp\left(-\frac{t^{2}}{16H^{2}T}\right).
	\]
	Setting $t=\sqrt{16H^{2}T\cdot\log(4/\delta)}$, we have 
	\[
	\left|M_{K,H}\right|\le\sqrt{16H^{2}T\cdot\log(4/\delta)}
	\]
	with probability at least $1-\delta/2$.
\end{proof}

\subsection{Proof of Lemma \ref{lem:dyn_reg_decomp}\label{sec:proof_dyn_reg_decomp}}

For any function $f:\cS\times\cA\to\real$ and any $(k,h,s)\in[K]\times[H]\times\cS$,
define the operators 
\[
(\J_{k,h}^{*}f)(s)=\left\langle f(s,\cdot),\pi_{h}^{*,k}(\cdot\ |\ s)\right\rangle ,\qquad(\J_{k,h}f)(s)=\left\langle f(s,\cdot),\pi_{h}^{k}(\cdot\ |\ s)\right\rangle .
\]
and the function 
\[
\xi_{h}^{k}(s)\coloneqq(\J_{k,h}^{*}Q_{h}^{k})(s)-(\J_{k,h}Q_{h}^{k})(s)=\left\langle Q_{h}^{k}(s,\cdot),\pi_{h}^{*,k}(\cdot\ |\ s)-\pi_{h}^{k}(\cdot\ |\ s)\right\rangle .
\]
The proof mostly follows that of \citet[Lemma 4.2]{cai2019provably},
except that we replace $\pi^{*}$ and $\J_{h}$ therein by $\pi^{*,k}$
and $\J_{k,h}^{*}$, respectively. Therefore, we outline the key steps
only and refer the readers to the proof of \citet[Lemma 4.2]{cai2019provably}
for full details. 

Recall that $\pi^{*,k}$ is the optimal policy in episode $k$. We
have

\begin{align*}
\dreg(K) & =\sum_{k\in[K]}\left[V_{1}^{\pi^{*,k},k}(s_{1}^{k})-V_{1}^{\pi^{k},k}(s_{1}^{k})\right]\\
& =\sum_{l\in[L]}\sum_{k=(l-1)\plen+1}^{l\plen}\left[V_{1}^{\pi^{*,k},k}(s_{1}^{k})-V_{1}^{\pi^{k},k}(s_{1}^{k})\right].
\end{align*}
We have 
\begin{equation}
V_{1}^{\pi^{*,k},k}(s_{1}^{k})-V_{1}^{\pi^{k},k}(s_{1}^{k})=\underbrace{V_{1}^{\pi^{*,k},k}(s_{1}^{k})-V_{1}^{k}(s_{1}^{k})}_{G_{1}}+\underbrace{V_{1}^{k}(s_{1}^{k})-V_{1}^{\pi^{k},k}(s_{1}^{k})}_{G_{2}}.\label{eq:V_decomp}
\end{equation}
From \citet[Section B.1]{cai2019provably}, we have for any $k\in[K]$,
\begin{align}
G_{1} & =\sum_{h\in[H]}\E_{\pi^{*,k}}[\iota_{h}^{k}(s_{h},a_{h})\ |\ s_{1}=s_{1}^{k}]\nonumber \\
& \quad+\sum_{h\in[H]}\E_{\pi^{*,k}}\left[\left\langle Q_{h}^{k}(s_{h},\cdot),\pi_{h}^{*,k}(\cdot\ |\ s_{h})-\pi_{h}^{k}(\cdot\ |\ s_{h})\right\rangle \ \Bigg|\ s_{1}=s_{1}^{k}\right],\label{eq:dyn_reg_decomp_G1}
\end{align}
and
\begin{equation}
G_{2}=-\sum_{h\in[H]}\iota_{h}^{k}(s_{h}^{k},a_{h}^{k})+\sum_{h\in[H]}(D_{h,1}^{k}+D_{h,2}^{k}),\label{eq:dyn_reg_decomp_G2}
\end{equation}
where 
\begin{align*}
D_{h,1}^{k} & \coloneqq\left(\J_{k,h}(Q_{h}^{k}-Q_{h}^{\pi^{k},k})\right)(s_{h}^{k})-(Q_{h}^{k}-Q_{h}^{\pi^{k},k})(s_{h}^{k},a_{h}^{k}),\\
D_{h,2}^{k} & \coloneqq\left(\P_{h}(V_{h+1}^{k}-V_{h+1}^{\pi^{k},k})\right)(s_{h}^{k},a_{h}^{k})-(V_{h+1}^{k}-V_{h+1}^{\pi^{k},k})(s_{h+1}^{k}).
\end{align*}
From Line \ref{line:dyn_oppo_policy_eval} of Algorithm \ref{alg:dyn_oppo},
we have 
\[
Q_{h}^{k},Q_{h}^{\pi^{k},k},V_{h+1}^{k},V_{h+1}^{\pi^{k},k}\in[0,H],
\]
which implies $\left|D_{h,1}^{k}\right|,\left|D_{h,2}^{k}\right|\le2H$
for any $(k,h)\in[K]\times[H]$. Writing $M_{h}^{k}\coloneqq D_{h,1}^{k}+D_{h,2}^{k}$,
we have that 
\[
M_{K,H}\coloneqq\sum_{k\in[K]}\sum_{h\in[H]}M_{h}^{k}
\]
is a martingale where $\left|M_{h}^{k}\right|\le4H$. The proof is
completed in view of Equations \eqref{eq:V_decomp}, \eqref{eq:dyn_reg_decomp_G1}
and \eqref{eq:dyn_reg_decomp_G2}.

\section{Proof of Theorem \ref{thm:regret_dyn_oppo} \label{sec:proof_regret_dyn_oppo}}

By Lemmas \ref{lem:bonus_ucb} and \ref{lem:bonus_sum_bound}, we
have 
\begin{align}
& \quad\sum_{l\in[\pnum]}\sum_{k=(l-1)\plen+1}^{l\plen}\sum_{h\in[H]}\left[\E_{\pi^{*,k}}[\iota_{h}^{k}(s_{h},a_{h})\ |\ s_{1}=s_{1}^{k}]-\iota_{h}^{k}(s_{h}^{k},a_{h}^{k})\right]\nonumber \\
& \le2\sum_{k\in[K[}\sum_{h\in[H]}\Gamma_{h}^{k}(s_{h}^{k},a_{h}^{k})\le2\beta H\sqrt{2dK\log((K+\lambda)/\lambda)}.\label{eq:model_pred_err_bound}
\end{align}
We apply Lemmas \ref{lem:perf_diff_bound_visit_meas} and \ref{lem:mtg_bound}
as well as Equation \eqref{eq:model_pred_err_bound} to the conclusion
of Lemma \ref{lem:dyn_reg_decomp}. With the choice of $\lambda=1$
and $\beta=C_{\beta}H\sqrt{S\log(dT/\delta)}$ and the identity $K=\pnum\plen$,
we have 
\begin{align}
\dreg(K) & \le2H^{2}\sqrt{KL\log A}+C\cdot\plen H^{2}\polchange+2\beta H\sqrt{2dK\log((K+\lambda)/\lambda)}\nonumber \\
& \quad+\sqrt{16H^{2}T\cdot\log(4/\delta)}\nonumber \\
& \le2H^{2}\sqrt{KL\log A}+C\cdot\plen H^{2}\polchange+2C_{\beta}H^{2}\sqrt{S\log(dT/\delta)}\sqrt{2dK\log(K+1)}\nonumber \\
& \quad+\sqrt{16H^{2}T\cdot\log(4/\delta)}\nonumber \\
& =2H^{2}\sqrt{KL\log A}+C\cdot\plen H^{2}\polchange+2C_{\beta}\sqrt{2H^{3}S^{2}AT\cdot\log(dT/\delta)\cdot\log(K+1)}\nonumber \\
& \quad+\sqrt{16H^{2}T\cdot\log(4/\delta)}\nonumber \\
& \le2H^{2}\sqrt{KL\log A}+C\cdot\plen H^{2}\polchange+C'\sqrt{2H^{3}S^{2}AT\cdot\log^{2}(dT/\delta)}\label{eq:dyn_oppo_final_regret}
\end{align}
where $C,C'>0$ are universal constants, the second step above holds
by the definition of $\beta$, and the third step holds by the identity
$T=KH$. 

We discuss several cases. 
\begin{itemize}
	\item If $0\le\polchange\le\sqrt{\frac{\log A}{K}}$, then by elementary
	calculation we have $\left(\frac{T\sqrt{\log A}}{H\polchange}\right)^{2/3}\ge K$.
	This implies $\plen=K$ by our choice of $\plen$, and therefore $\pnum=1$.
	Then Equation \eqref{eq:dyn_oppo_final_regret} yields 
	\begin{align*}
	\dreg(K) & \le2H^{2}\sqrt{K\log A}+C\cdot H^{2}\sqrt{K\log A}+C'\sqrt{2H^{3}S^{2}AT\log^{2}(dT/\delta)}\\
	& =(2+C)\sqrt{H^{3}T\log A}+C'\sqrt{2H^{3}S^{2}AT\log^{2}(dT/\delta)}.
	\end{align*}
	\item If $\sqrt{\frac{\log A}{K}}\le\polchange\le2^{-3/2}\cdot K\sqrt{\log A}$,
	we have and $2\le\plen\le K$ and Equation \eqref{eq:dyn_oppo_final_regret}
	implies 
	\begin{align*}
	\dreg(K) & \le2\cdot\frac{1}{\sqrt{\plen}}HT\sqrt{\log A}+C\cdot\plen H^{2}\polchange+C'\sqrt{2H^{3}S^{2}AT\cdot\log^{2}(dT/\delta)}\\
	& \le(4+C)\cdot\left(H^{2}T\sqrt{\log A}\right)^{2/3}\polchange^{1/3}+C'\sqrt{2H^{3}S^{2}AT\cdot\log^{2}(dT/\delta)},
	\end{align*}
	where the first step holds by $K=\pnum\plen$, and in the last step
	we applied the choice of $\tau=\left\lfloor \left(\frac{T\sqrt{\log A}}{H\polchange}\right)^{2/3}\right\rfloor $. 
	\item If $\polchange>2^{-3/2}\cdot K\sqrt{\log A}$, we have $\left(\frac{T\sqrt{\log A}}{H\polchange}\right)^{2/3}<2$
	and therefore $\plen=1$ and $\pnum=K$. Then Equation \eqref{eq:dyn_oppo_final_regret}
	yields 
	\begin{align*}
	\dreg(K) & \le2HT\sqrt{\log A}+C\cdot H^{2}\polchange+C'\sqrt{2H^{3}S^{2}AT\cdot\log^{2}(dT/\delta)}\\
	& \le(8+C)H^{2}\polchange+C'\sqrt{2H^{3}S^{2}AT\cdot\log^{2}(dT/\delta)},
	\end{align*}
\end{itemize}
It is not hard to see that all of the above arguments also go through
if we replace $\polchange$ with its upper bound. The proof is completed
by combining the last case above with the trivial bound $\dreg(K)\le T$.

\section{Proof of Theorem \ref{thm:regret_acc_oppo}\label{sec:proof_regret_acc_oppo}}

The proof follows the same reasoning as in Appendix \ref{sec:proof_regret_dyn_oppo}, 
except that Lemmas \ref{lem:perf_diff_bound_visit_meas} no longer applies. 
In the following, we provide an alternative to Lemmas \ref{lem:perf_diff_bound_visit_meas} adapted for Algorithm \ref{alg:acc_oppo}.


\begin{lem}
	\emph{\label{lem:perf_diff_bound_fixed_state_acc_oppo} }For any $s\in\cS$,
	we have
	\begin{align*}
	& \quad\sum_{l\in[\pnum]}\sum_{k=(l-1)\plen+1}^{l\plen}\sum_{h\in[H]}\left\langle Q_{h}^{k}(s,\cdot),\pi_{h}^{*,k}(\cdot\ |\ s)-\pi_{h}^{k}(\cdot\ |\ s)\right\rangle \\
	& \le\alpha\Qchange+\frac{1}{\alpha}\pnum H\log A+\plen H\polchange.
	\end{align*}
\end{lem}
\begin{proof}
	Let us fix an $s\in\cS$. For each $l\in[\pnum]$, we let $\nu^{l}=\{\nu_{h}^{l}\}_{h\in[H]}$
	where each $\nu_{h}^{l}$ is a policy (or a distribution supported
	$\cA$) that depends only on $l$ and $h$. We have the decomposition 
	\begin{align*}
	& \quad\sum_{l\in[\pnum]}\sum_{k=(l-1)\plen+1}^{l\plen}\sum_{h\in[H]}\left\langle Q_{h}^{k}(s,\cdot),\pi_{h}^{*,k}(\cdot\ |\ s)-\pi_{h}^{k}(\cdot\ |\ s)\right\rangle \\
	& \le\sum_{l\in[\pnum]}\sum_{k=(l-1)\plen+1}^{l\plen}\sum_{h\in[H]}\left\langle Q_{h}^{k}(s,\cdot),\nu_{h}^{l}(\cdot\ |\ s)-\pi_{h}^{k}(\cdot\ |\ s)\right\rangle \\
	& \quad+\sum_{l\in[\pnum]}\sum_{k=(l-1)\plen+1}^{l\plen}\sum_{h\in[H]}\left\langle Q_{h}^{k}(s,\cdot),\pi_{h}^{*,k}(\cdot\ |\ s)-\nu_{h}^{l}(\cdot\ |\ s)\right\rangle \\
	& \eqqcolon E_{1}+E_{2}.
	\end{align*}
	The term $E_{2}$ can be controlled in exactly the same way as in
	the proof of Lemma \ref{lem:perf_diff_bound_fixed_expect}. Therefore,
	we only control $E_{1}$. Note that the policy update steps in
	Algorithm \ref{alg:acc_oppo} (Lines \ref{line:acc_oppo_update_policy_mirror}
	and \ref{line:acc_oppo_update_policy}) essentially follow the update
	steps of OMD (see e.g.$\ $\citet[Section 3.1.1]{syrgkanis2015fast}
	for details). This observation enables us to take advantage of the
	following lemma, which is a version of \citet[Proposition 5]{syrgkanis2015fast}
	adapted to our case.
	\begin{lem}
		\label{lem:fast_convg_prop_5}For any $(l,h,s)\in[\pnum]\times[H]\times\cS$,
		we have 
		\begin{align*}
		& \quad\sum_{k=(l-1)\plen+1}^{l\plen}\left\langle Q_{h}^{k}(s,\cdot),\nu_{h}^{l}(\cdot\ |\ s)-\pi_{h}^{k}(\cdot\ |\ s)\right\rangle \\
		& \le\frac{\log A}{\alpha}+\alpha\cdot\sum_{k=(l-1)\plen+1}^{l\plen}\norm[Q_{h}^{k}(s,\cdot)-Q_{h}^{k-1}(s,\cdot)]{\infty}^{2}-\frac{1}{8\alpha}\cdot\sum_{k=(l-1)\plen+1}^{l\plen}\norm[\pi_{h}^{k}-\pi_{h}^{k-1}]{\infty}^{2}.
		\end{align*}
	\end{lem}
	\begin{proof}
		The result follows from \citet[Proposition 5]{syrgkanis2015fast}
		and we note that the quantity $R$ defined therein is upper bounded
		by $\log A$.
	\end{proof}
	By Lemma \ref{lem:fast_convg_prop_5} and the definition of $\Qchange$
	in Equation \eqref{eq:Q_change}, we have 
	\begin{align*}
	E_{1} & \le\pnum\cdot H\cdot\frac{\log A}{\alpha}+\alpha\cdot\Qchange,
	\end{align*}
	We have the following result on the performance difference,
	similar to Lemma \ref{lem:perf_diff_bound_visit_meas}.
\end{proof}
\begin{lem}
	\label{lem:perf_diff_bound_visit_meas_acc_oppo}Recall that $\polchange\coloneqq\sum_{k\in[K]}\sum_{i\in[H]}\norm[\pi_{i}^{*,k}-\pi_{i}^{*,k-1}]{\infty}$.
	Under Assumption \ref{asp:visit_measure_smooth}, we choose $\alpha=\sqrt{\frac{\pnum H\log A}{\Qchange}}$
	in Algorithm \ref{alg:acc_oppo}, and we have 
	\begin{align*}
	& \quad\sum_{l\in[\pnum]}\sum_{k=(l-1)\plen+1}^{l\plen}\sum_{h\in[H]}\E_{\pi^{*,k}}\left[\left\langle Q_{h}^{k}(s_{h},\cdot),\pi_{h}^{*,k}(\cdot\ |\ s_{h})-\pi_{h}^{k}(\cdot\ |\ s_{h})\right\rangle \ \Bigg|\ s_{1}=s_{1}^{k}\right]\\
	& =2\sqrt{\Qchange\pnum H\log A}+C\cdot\plen H^{2}\polchange,
	\end{align*}
	for some universal constant $C>0$.
\end{lem}
\begin{proof}
	Now, for any $l\in[\pnum]$ we have the decomposition
	\begin{align*}
	& \quad\sum_{l\in[\pnum]}\sum_{k=(l-1)\plen+1}^{l\plen}\sum_{h\in[H]}\E_{\pi^{*,k}}\left[\left\langle Q_{h}^{k}(s_{h},\cdot),\pi_{h}^{*,k}(\cdot\ |\ s_{h})-\pi_{h}^{k}(\cdot\ |\ s_{h})\right\rangle \ \Bigg|\ s_{1}=s_{1}^{k}\right]\\
	& =\sum_{l\in[\pnum]}\sum_{k=(l-1)\plen+1}^{l\plen}\sum_{h\in[H]}\E_{\pi^{*,(l-1)\plen+1}}\left[\left\langle Q_{h}^{k}(s_{h},\cdot),\pi_{h}^{*,k}(\cdot\ |\ s_{h})-\pi_{h}^{k}(\cdot\ |\ s_{h})\right\rangle \ \Bigg|\ s_{1}=s_{1}^{k}\right]\\
	& \quad+\sum_{l\in[\pnum]}\sum_{k=(l-1)\plen+1}^{l\plen}\sum_{h\in[H]}\left(\E_{\pi^{*,k}}-\E_{\pi^{*,(l-1)\plen+1}}\right)\left[\left\langle Q_{h}^{k}(s_{h},\cdot),\pi_{h}^{*,k}(\cdot\ |\ s_{h})-\pi_{h}^{k}(\cdot\ |\ s_{h})\right\rangle \ \Bigg|\ s_{1}=s_{1}^{k}\right]
	\end{align*}
	By applying Lemmas \ref{lem:perf_diff_bound_fixed_state_acc_oppo}
	and \ref{lem:perf_diff_bound_varying_policies}, we have 
	\begin{align*}
	& \quad\sum_{l\in[\pnum]}\sum_{k=(l-1)\plen+1}^{l\plen}\sum_{h\in[H]}\E_{\pi^{*,k}}\left[\left\langle Q_{h}^{k}(s_{h},\cdot),\pi_{h}^{*,k}(\cdot\ |\ s_{h})-\pi_{h}^{k}(\cdot\ |\ s_{h})\right\rangle \ \Bigg|\ s_{1}=s_{1}^{k}\right]\\
	& \le\alpha\Qchange+\frac{\pnum H\log A}{\alpha}+\plen H\polchange+C'\cdot\plen H^{2}\polchange\\
	& =2\sqrt{\Qchange\pnum H\log A}+\plen H\polchange+C'\cdot\plen H^{2}\polchange\\
	& \le2\sqrt{\Qchange\pnum H\log A}+C\cdot\plen H^{2}\polchange,
	\end{align*}
	where the last equality holds by our choice of $\alpha$, and $C,C'>0$
	are universal constants
\end{proof}
We apply Lemmas \ref{lem:perf_diff_bound_visit_meas_acc_oppo} and
\ref{lem:mtg_bound} and Equation \eqref{eq:model_pred_err_bound}
to the conclusion of Lemma \ref{lem:dyn_reg_decomp}. With the choice
of $\lambda=1$ and $\beta=C_{\beta}H\sqrt{S\log(dT/\delta)}$, we
have
\begin{align}
\dreg(K) & \le2\sqrt{\Qchange\pnum H\log A}+C\cdot\plen H^{2}\polchange+2\beta H\sqrt{2dK\log((K+\lambda)/\lambda)}\nonumber \\
& \quad+\sqrt{16H^{2}T\cdot\log(4/\delta)}\nonumber \\
& \le2\sqrt{\Qchange\pnum H\log A}+C\cdot\plen H^{2}\polchange+2C_{\beta}H^{2}\sqrt{S\log(dT/\delta)}\sqrt{2dK\log(K+1)}\nonumber \\
& \quad+\sqrt{16H^{2}T\cdot\log(4/\delta)}\nonumber \\
& =2\sqrt{\Qchange\pnum H\log A}+C\cdot\plen H^{2}\polchange+2C_{\beta}\sqrt{2H^{3}S^{2}AT\cdot\log(dT/\delta)\cdot\log(K+1)}\nonumber \\
& \quad+\sqrt{16H^{2}T\cdot\log(4/\delta)}\nonumber \\
& \le2\sqrt{\Qchange\pnum H\log A}+C\cdot\plen H^{2}\polchange+C'\sqrt{2H^{3}S^{2}AT\cdot\log^{2}(dT/\delta)}\label{eq:acc_oppo_final_regret}
\end{align}
where $C,C'>0$ are universal constants, the second step holds by
the definition of $\beta$, and the third step holds by the identity
$T=KH$. Analyzing Equation \eqref{eq:acc_oppo_final_regret} in
the same way as Equation \eqref{eq:dyn_oppo_final_regret} in Section
\ref{sec:proof_regret_dyn_oppo} (for different regimes of $\polchange$)
yields the result.

\end{document}